\documentclass{article}

\usepackage{microtype}
\usepackage{graphicx}
\usepackage{subcaption}
\usepackage{booktabs} 

\usepackage{hyperref}

\usepackage{amsmath}
\usepackage{amssymb}
\usepackage{mathtools}
\usepackage{amsthm}
\usepackage{multirow}
\usepackage{xcolor}   
\usepackage{colortbl}   
\definecolor{lightblue}{HTML}{E6E6FA}
\definecolor{lightgreen}{HTML}{D4EDDA}
\allowdisplaybreaks
\usepackage[capitalize,noabbrev]{cleveref}
\usepackage[accepted]{icml2026}
\theoremstyle{plain}
\newtheorem{theorem}{Theorem}[section]
\newtheorem{proposition}[theorem]{Proposition}

\newtheorem{corollary}[theorem]{Corollary}
\theoremstyle{definition}

\theoremstyle{remark}

\usepackage[textsize=tiny]{todonotes}
\icmltitlerunning{Boosting Maximum Entropy Reinforcement Learning via One-Step Flow Matching}
\begin{document}

\twocolumn[
  \icmltitle{Boosting Maximum Entropy Reinforcement Learning \\ via One-Step Flow Matching}



  \icmlsetsymbol{equal}{*}

\begin{icmlauthorlist}
\icmlauthor{Zeqiao Li}{yyy}
\icmlauthor{Yijing Wang}{yyy}
\icmlauthor{Haoyu Wang}{yyy}
\icmlauthor{Zheng Li}{yyy}
\icmlauthor{Zhiqiang Zuo}{yyy}

\end{icmlauthorlist}

\icmlaffiliation{yyy}{Tianjin Key Laboratory of Intelligent Unmanned Swarm Technology and System, School of Electrical and Information Engineering, Tianjin University, Tianjin 300072, China}

\icmlcorrespondingauthor{Haoyu Wang}{why2014@tju.edu.cn}

  \vskip 0.3in
]

\printAffiliationsAndNotice{}  

\begin{abstract}
Diffusion policies are expressive yet incur high inference latency. Flow Matching (FM) enables one-step generation, but integrating it into Maximum Entropy Reinforcement Learning (MaxEnt RL) is challenging: the optimal policy is an intractable energy-based distribution, and the efficient log-likelihood estimation required to balance exploration and exploitation suffers from severe discretization bias. We propose \textbf{F}low-based \textbf{L}og-likelihood-\textbf{A}ware \textbf{M}aximum \textbf{E}ntropy RL (\textbf{FLAME}), a principled framework that addresses these challenges. First, we derive a Q-Reweighted FM objective that bypasses partition function estimation via importance reweighting. Second, we design a decoupled entropy estimator that rigorously corrects bias, which enables efficient exploration and brings the policy closer to the optimal MaxEnt policy. Third, we integrate the MeanFlow formulation to achieve expressive and efficient one-step control. Empirical results on MuJoCo show that FLAME outperforms Gaussian baselines and matches multi-step diffusion policies with significantly lower inference cost. Code is available at \url{https://github.com/lzqw/FLAME}.
\end{abstract}

\section{Introduction}
\label{sec:introduction}

Recent advances in continuous control have been driven by expressive generative policy representations. Approaches based on diffusion models (DM)~\cite{DP1-DDPM,DP3-DDIM} and flow matching (FM)~\cite{lipman2023flowmatchinggenerativemodeling,liu2022flow,geng2025mean} parameterize complex action distributions by reversing probability flows in continuous spaces. Compared to unimodal Gaussian policies, these models capture rich multi-modal behaviors and have shown strong empirical performance across diverse control tasks~\cite{DP-APPLI1-chi2024diffusion,DP-POLOCY4-pi0,DP-POLOCY3-Octo}.

Despite their expressiveness, integrating generative policies into online Maximum Entropy Reinforcement Learning (MaxEnt RL) remains challenging. Algorithms such as Soft Actor-Critic (SAC)~\cite{haarnoja2018softactorcriticoffpolicymaximum,haarnoja2019softactorcriticalgorithmsapplications} rely on simple Gaussian parameterizations to ensure tractable policy optimization and entropy regularization. Extending diffusion models to online RL~\cite{ma2025effi,wang2024diffusion,ding2024diffusion} introduces a fundamental bottleneck: action generation requires a large number of function evaluations (NFE), resulting in inference latency that limits high-frequency control. Although distillation-based acceleration technique have been explored~\cite{ding2024consistencymodelsrichefficient}, it typically introduces auxiliary objectives or multi-stage pipelines that complicate training and deployment.

FM enables high-fidelity one-step action generation by learning transport-inspired probability paths that can straighten as the target distribution concentrates~\cite{hu2024adaflow}, making it a practical alternative to multi-step diffusion policies. Meanwhile, online MaxEnt RL offers a principled objective that balances return maximization and exploration via entropy regularization. Bridging the two, however, remains difficult for two reasons. First, flow matching is inherently target-sample driven, whereas MaxEnt policy improvement is defined by an implicit, unnormalized energy-based target induced by a continually evolving critic, so the regression targets required by FM are not directly available online. Second, MaxEnt RL depends on accurate entropy terms, but computing reliable likelihoods for continuous-time flows is expensive in the online loop, and coarse discretizations can introduce bias that destabilizes value learning.

Consequently, existing flow-based RL methods typically obtain fast sampling by avoiding explicit likelihood-based MaxEnt optimization. One line uses surrogate policy-improvement objectives that are compatible with flow training, such as mirror-descent style targets in FPMD~\citep{chen2025onestepflowpolicymirror} or transport-regularized policy search in FlowRL~\citep{lv2025flowbasedpolicyonlinereinforcement}; these yield efficient one-step execution but do not optimize the MaxEnt entropy term through policy likelihood. Another line treats flow policies as expressive actors and fine-tunes them with on-policy surrogates, e.g., ReinFlow~\citep{zhang2025reinflowfinetuningflowmatching}, which improves stability via a discrete-time stochastic formulation yet remains centered on PPO-style objectives rather than general MaxEnt policy iteration. In contrast, MEOW~\citep{chao2024} enables exact likelihood computation by adopting an explicit energy-based normalizing-flow design with tractable normalization, at the cost of additional architectural constraints and energy parameterization

To bridge FM with the MaxEnt RL framework, we propose \textbf{F}low-based \textbf{L}og-likelihood-\textbf{A}ware \textbf{M}aximum \textbf{E}ntropy RL (\textbf{FLAME}). FLAME resolves the tension between expressive flow-based policies and the MaxEnt objective through a principled integration of distribution matching and entropy regularization. For one thing, we bypass the intractability of energy-based targets via a Q-Reweighted objective that learns from value estimates implicitly. For another, we reconcile the conflict between efficient inference and accurate density estimation through a decoupled strategy, ensuring that the exploration benefits of MaxEnt are preserved without incurring deployment latency. Our contributions are summarized as follows:

\textbf{Implicit MaxEnt Policy Learning via Q-Reweighting.} We derive a Q-Reweighted FM objective that leverages reverse sampling to analytically cancel the partition function. This allows implicit MaxEnt policy optimization directly from soft Q-values, without relying on surrogate objectives.

\textbf{Decoupled Entropy Estimation for Efficient Exploration.} We address instability from discretization bias via a decoupled strategy employing multi-step integration during training. This restores accurate entropy regularization, preventing mode collapse and enabling diverse, multi-modal exploration.

\textbf{Expressive and Efficient One-Step Control.} By integrating MeanFlow, FLAME achieves high-fidelity one-step action generation (NFE=1). This effectively bridges the gap between generative policy expressivity and the strict low-latency requirements of real-time control.

\section{Preliminaries} 

\subsection{Maximum Entropy Reinforcement Learning}
We study an infinite-horizon Markov decision process
$\mathcal{M}=(\mathcal{S},\mathcal{A},r,P,\mu_0,\gamma)$, where $\mathcal{S}$ and
$\mathcal{A}$ denote the state and action spaces, $r:\mathcal{S}\times\mathcal{A}\to\mathbb{R}$
is the reward function, $P: \mathcal{S} \times \mathcal{A} \rightarrow \Delta(\mathcal{S})$ is the transition operator with $\Delta(\mathcal{S})$ being the family of distributions over $\mathcal{S}$, $\mu_0 \in \Delta(\mathcal{S})$ is the initial-state distribution, and $\gamma\in(0,1)$ is the discount factor.
Maximum Entropy RL augments return maximization with an entropy bonus and seeks a stochastic policy by
\begin{equation}
\begin{aligned}
\pi^* \!:=& \arg\max_\pi J(\pi) \\
=& \mathbb{E}_\pi \! \Big[ \sum_{\tau=0}^{\infty} \gamma^\tau \big( r(s_\tau, a_\tau) + \alpha \mathcal{H}(\pi(\cdot|s_\tau)) \big) \Big],
\end{aligned}
\end{equation}
where $\mathcal{H}(\pi(\cdot\mid s))=\mathbb{E}_{a\sim\pi(\cdot\mid s)}\!\left[-\log\pi(a\mid s)\right]$
is the entropy, and $\alpha>0$ is a regularization coefficient for the entropy.

A standard approach is soft policy iteration~\cite{haarnoja2018softactorcriticoffpolicymaximum}, which alternates soft policy evaluation and improvement. For soft policy evaluation, it updates the soft Q-function by repeatedly applying the soft Bellman update operator $\mathcal{T}^\pi$ to the current value function $Q: \mathcal{S} \times \mathcal{A} \rightarrow \mathbb{R}$, i.e.,
\begin{equation}
\label{eq: bellman operator}
\mathcal{T}^\pi Q\left(s_\tau, a_\tau\right)=r\left(s_\tau, a_\tau\right)+\gamma \mathbb{E}_{s_{\tau+1} \sim P}\left[V\left(s_{\tau+1}\right)\right]
\end{equation}
where $V\left(s_t\right)=\mathbb{E}_{a_t \sim \pi}\left[Q\left(s_t, a_t\right)-\alpha \log \pi\left(a_t \mid s_t\right)\right]$ is the soft value function.  Given $Q^{\pi_{\text{old}}}(s,a)$ as the converged result of \eqref{eq:weighted_fm_prop}, the improvement step targets the Boltzmann (energy-based) policy
\begin{equation}
\label{eq: max-ent policy}
\begin{split}
\pi_{\mathrm{MaxEnt}}(a\mid s)
&=\frac{\exp\!\left(Q^{\pi_{\text{old}}}(s,a)/\alpha\right)}{Z(s)}, \\
Z(s)&=\int \exp\!\left(Q^{\pi_{\text{old}}}(s,a)/\alpha\right)\,\mathrm{d}a.
\end{split}
\end{equation}
This formulation provides a principled exploration--exploitation trade-off and underpins practical algorithms such as SAC~\cite{haarnoja2019softactorcriticalgorithmsapplications}.

However, the normalizer $Z(s)$ is generally intractable, hence $\pi_{\mathrm{MaxEnt}}$ is an energy-based model (EBM).
In general, an EBM takes the form
\begin{equation}
p(\mathbf{x})=\frac{\exp(-E(\mathbf{x}))}{Z},
\end{equation}
where $Z=\int \exp(-E(\mathbf{x}))\,\mathrm{d}\mathbf{x}$ is the partition function.
In online RL, a common workaround is to project the EBM target onto a tractable policy class.
For instance, SAC parameterizes a Gaussian policy
$\pi_\theta(a \mid s)=\mathcal{N}\left(\mu_{\theta_1}(s), \sigma_{\theta_2}^2(s)\right)$
and updates the parameters $\theta=\left[\theta_1, \theta_2\right]$ by minimizing the  Kullback-Leible (KL) divergence to the MaxEnt target,
$\min_\theta D_{\mathrm{KL}}(\pi_\theta \,\|\, \pi_{\mathrm{MaxEnt}})$.
While effective, this projection restricts expressivity and can yield suboptimal policies when the optimal MaxEnt distribution is multi-modal.

\subsection{Flow Matching and Mean Flow}
\label{sec: Flow Matching and Mean Flow}

FM is a generative framework that learns a continuous-time transport from a source distribution $p_0$ to a target distribution $p_1$ via an ordinary differential equation (ODE)~\citep{lipman2023flowmatchinggenerativemodeling}. In the context of RL, the flow acts on the action space $\mathcal{A} \subseteq \mathbb{R}^d$, where $d$ denotes the dimension of the action, conditioned on the state $s \in \mathcal{S}$.\footnote{In this section, $t\in[0,1]$ denotes the continuous flow time. The intermediate variable is written as an action $a_t$ rather than $x_t$ to match the policy notation used later. The symbol $\tau$ is reserved for discrete RL time steps.}

Formally, consider a time-dependent probability density path $p_t: \mathbb{R}^d \times \mathcal{S} \rightarrow \mathbb{R}_{>0}$ for $t \in [0,1]$. This path is generated by a time-dependent vector field $u_t: \mathbb{R}^d \times \mathcal{S} \rightarrow \mathbb{R}^d$, which defines a flow map $\psi_t: \mathbb{R}^d \times \mathcal{S} \rightarrow \mathbb{R}^d$ via the following ODE:
\begin{align}
    \frac{d}{dt}\psi_t(a_0 \mid s) = u_t(\psi_t(a_0 \mid s) \mid s), \psi_0(a_0 \mid s) = a_0.
    \label{eq:ode}
\end{align}
A flow policy $\pi(\cdot\mid s)$ is defined by sampling $a_0 \sim p_0(\cdot)$ and integrating Eq.~\eqref{eq:ode} to $t=1$, yielding the terminal action $a_1 = \psi_1(a_0 \mid s)$. To learn the dynamics, a neural network $u_\theta(a_t, t, s)$ with parameters $\theta$ is used to approximate the vector field $u_t(a_t \mid s)$. The marginal FM objective $\mathcal{L}_{\mathrm{FM}}(\theta)$ minimizes the regression error against the target vector field $u_t$ over the probability path:
\begin{equation}
    \label{eq:fm_loss}
    \mathcal{L}_{\mathrm{FM}}(\theta)
    = \mathbb{E}_{t, s, a_t} \!\left[ \left\| u_{\theta}(a_t, t, s) - u_t(a_t \mid s) \right\|^2 \right].
\end{equation}
where $t \sim \mathcal{U}[0,1]$, $s \sim \rho$ and  $a_t \sim p_t(\cdot\mid s)$. Directly regressing the marginal vector field is often intractable. In practice, Conditional Flow Matching (CFM) object $\mathcal{L}_{\mathrm{CFM}}(\theta)$ is adopted, which replaces the marginal target with a tractable conditional vector field defined given samples $a_1 \sim \pi(\cdot\mid s)$:
\begin{equation}
    \label{eq:cfm_loss}
    \mathcal{L}_{\mathrm{CFM}}(\theta)
    = \mathbb{E}_{t, s, a_1, a_t} \!\left[
    \left\| 
    u_{\theta}(a_t, t, s) - u_t(a_t \mid a_1) 
    \right\|^2 
    \right],
\end{equation}
We adopt the Optimal-Transport (OT) linear coupling $a_t = t a_1 + (1-t)a_0$, which yields a straight conditional target field:
\begin{equation}
    \label{eq:ot_vector_field}
    u_t(a_t \mid a_1) = a_1 - a_0.
\end{equation}

While the conditional paths are straight, the induced marginal paths may still be curved, requiring multi-step integration. To address this, MeanFlow~\citep{geng2025mean} introduces a theoretical framework to rectify the vector field by learning the average velocity along the flow trajectories. Formally, the average velocity field $\overline{u}_{t,\zeta}$ is defined as:
\begin{equation}
    \label{eq:meanflow_avg}
    \overline{u}_{t,\zeta}(a_t \mid s) \triangleq \frac{1}{t-\zeta}\int_\zeta^t u_\xi(a_\xi \mid a_1) \mathrm{d}\xi,
\end{equation}
where $\zeta \in [0, t)$ is the starting time. It is trained via a variational iteration loss
\begin{equation}
    \label{eq:meanflow_objective}
    \mathcal{L}_{\mathrm{MF}}(\theta)
    = \mathbb{E}_{{t, \zeta, s, a_1, a_t}} \!\left[
    \left\| 
    \overline{u}_\theta(a_t, \zeta, t, s) - \text{sg}(\overline{u}_{\text{tgt}})
    \right\|^2 
    \right],
\end{equation}
where $\text{sg}(\cdot)$ denotes the stop-gradient operator. The target $\overline{u}_{\text{tgt}}$ is derived from the transport equation:
\begin{equation}
    \label{eq:meanflow_target}
    \overline{u}_{\text{tgt}}
    = u_t(a_t \mid a_1) - (t-\zeta)\left( \partial_a \overline{u}_\theta \cdot u_t(a_t \mid a_1) + \partial_t \overline{u}_\theta \right).
\end{equation}
This formulation provides a tractable objective to learn a velocity field with minimal curvature, laying the mathematical foundation for efficient trajectory integration.

\section{Methodology}
\label{sec:methodology}
In this section, we present FLAME, a framework that integrates FM into the online MaxEnt RL paradigm. As defined in Sec.~\ref{sec: Flow Matching and Mean Flow}, our policy $\pi_\theta(a \mid s)$ is given by the terminal distribution of a state-conditioned flow. However, directly optimizing flow policies under the MaxEnt objective encounters two fundamental obstacles. i) Sampling intractability in policy improvement: the MaxEnt policy improvement step yields an energy-based target policy defined in Eq.~\eqref{eq: max-ent policy}, whose normalizer $Z(s)$ is generally intractable. Consequently, direct sampling $a\sim \pi_{\text{new}}(\cdot\mid s)$ is not straightforward, which prevents forming the regression targets required by CFM. 
ii) Density intractability in policy evaluation: entropy regularization in off-policy MaxEnt RL requires computing $\log \pi_\theta(a \mid s)$ for critic targets. For continuous flows, the likelihood computation typically requires numerical integration. However, efficient one-step approximations introduce significant discretization bias, while accurate multi-step integration incurs prohibitive computational overhead during training.

To address these challenges, we first derive a Q-Reweighted FM objective that cancels the partition function via importance reweighting and enables tractable policy improvement without sampling from $\pi_{\text{new}}$. Next, we theoretically quantify the discretization bias in flow-based likelihood estimation and introduce two complementary strategies: FLAME-R employs a rigorous augmented-ODE formulation for unbiased evaluation, while FLAME-M utilizes a novel decoupled multi-step estimator to reconcile accurate entropy estimation with one-step inference. Finally, we instantiate FLAME into two variants, FLAME-R and FLAME-M.

\subsection{Q-Reweighted Flow Policy Learning}
\label{subsec:q_weighted_loss}

The objective of policy improvement is to train the flow policy to approximate the optimal energy-based distribution defined by the current soft Q-function. To achieve this within FM, we first establish the theoretical connection between the intractable marginal objective and the tractable conditional regression.

\begin{proposition}[Gradient Equivalence of FM and CFM]
\label{prop:fm_cfm_equivalence}
Assuming the marginal probability density satisfies $p_t(a \mid s) > 0$ for all $a \in \mathcal{A}$ and $t \in [0, 1]$, the conditional objective $\mathcal{L}_{\mathrm{CFM}}$ (Eq. \ref{eq:cfm_loss}) and the marginal objective $\mathcal{L}_{\mathrm{FM}}$ (Eq. \ref{eq:fm_loss}) are equivalent up to a constant independent of $\theta$:
\begin{equation}
    \mathcal{L}_{\mathrm{CFM}}(\theta) = \mathcal{L}_{\mathrm{FM}}(\theta) + C_1.
\end{equation}
Consequently, their gradients satisfy $\nabla_{\!\theta}\mathcal{L}_{\text{FM}} \!=\! \nabla_{\!\theta}\mathcal{L}_{\text{CFM}}$.
\end{proposition}
This fundamental connection was first revealed in \citet{lipman2023flowmatchinggenerativemodeling}, and we revisit it here to justify the use of conditional regression targets for policy learning. The proof is provided in Appendix \ref{subsec:derivations_prop}.

This equivalence justifies minimizing the conditional loss to optimize the marginal vector field. However, the MaxEnt target sample $a_1$ is drawn from the energy-based distribution defined in Eq.~\eqref{eq: max-ent policy}, which involves an intractable partition function $Z(s)$. To bypass the intractability of sampling from the energy-based target $\pi_{\text{new}}$, we generalize the standard flow matching objective via importance reweighting. Our approach leverages the insight that the flow matching objective can be integrated against any strictly positive weight function $g(a_t, s)$ without altering the optimal vector field solution. This property is formalized as follows:

\begin{proposition}[Reweighting Invariance]
\label{prop:reweighting_invariance}
Let $g: \mathcal{A} \times \mathcal{S} \to (0, \infty)$ be any strictly positive measurable function. Define the $g$-weighted marginal flow matching objective as:
\begin{equation}
    \label{eq:weighted_fm_prop}
    \mathcal{L}_{\mathrm{FM}}^{g}(\theta) \!\triangleq\! \mathbb{E}_{t,s}\!\!\left[\!\int_{\!\mathcal{A}} g(a_t,s)\big\|u_\theta(a_t,t,s)\!-\!u_t(a_t|s)\big\|^2 \mathrm{d}a_t\!\right]\!.
\end{equation}
Under the realizability assumption, the set of global minimizers of $\mathcal{L}_{\mathrm{FM}}^{g}(\theta)$ is identical to that of the unweighted objective $\mathcal{L}_{\mathrm{FM}}(\theta)$. Specifically, any global minimizer $\theta^*$ satisfies $u_{\theta^*}(a_t, t, s) = u_t(a_t \mid s)$ almost everywhere.
\end{proposition}

Leveraging this invariance, we construct a specific weighting function $g$ to eliminate the intractable partition function $Z(s)$. We first rewrite the marginal objective in Eq.~\eqref{eq:weighted_fm_prop} into its conditional form by decomposing the marginal density $p_t(a_t \mid s) = \int p_t(a_t \mid a_1) \pi_{\text{new}}(a_1 \mid s) \mathrm{d}a_1$:
\begin{equation}
\begin{aligned}
    \mathcal{L}_{\mathrm{FM}}^g(\theta) &= \iint \frac{g(a_t, s)}{p_t(a_t \mid s)} \, p_t(a_t \mid a_1) \, \pi_{\text{new}}(a_1 \mid s) \\
    &\quad \times \left\| u_{\theta}(a_t, t, s) - u_t(a_t \mid a_1) \right\|^2 \mathrm{d}a_1 \mathrm{d}a_t + C_2,
\end{aligned}
\end{equation}
where $C_2$ is a constant independent of $\theta$. We select the weighting function $g$ specifically to cancel both the partition function in $\pi_{\text{new}}$ and the intractable marginal density:
\begin{equation}
    g^{\text{MaxEnt}}(a_t, s) = h_t(a_t \mid s) \, Z(s) \, p_t(a_t \mid s),
\end{equation}
where $h_t(a_t \mid s)$ is a tractable proposal distribution (e.g., uniform over $\mathcal{A}$ or the previous policy $\pi_{\text{old}}$) with full support on $\mathcal{A}$. Substituting $g^{\text{MaxEnt}}$ and Eq.~\eqref{eq: max-ent policy} into the generalized loss cancels both $Z(s)$ and $p_t(a_t \mid s)$, yielding:
\begin{equation}
\begin{aligned}
    \mathcal{L}^{g}(\theta) &\propto \iint h_t(a_t \mid s) \, p_t(a_t \mid a_1) \, \exp\left(\frac{Q(s, a_1)}{\alpha}\right) \\
    &\quad \times \left\| u_{\theta}(a_t, t, s) - u_t(a_t \mid a_1) \right\|^2 \mathrm{d}a_1 \mathrm{d}a_t.
\end{aligned}
\end{equation}

Throughout, we exclude the endpoint $t=0$ and sample $t \sim \mathcal{U}[\varepsilon,1]$ (e.g., $\varepsilon=10^{-3}$) for numerical stability, since the reverse kernel involves factors of $1/t$.
To render the expectation over $p_t(a_t \mid a_1)$ tractable without sampling from the intractable $\pi_{\text{new}}$, we apply the Reverse Sampling Trick.
For the OT probability path $a_t = t a_1 + (1-t) a_0$ with Gaussian base $a_0 \sim p_0=\mathcal{N}(0,I)$,
the forward transition kernel is
$p_t(a_t \mid a_1) = \mathcal{N}(a_t \mid t a_1, (1-t)^2 I)$.
Define the reverse conditional distribution
\begin{equation}
    \phi_{1|t}(a_1 \mid a_t) = \mathcal{N}\left(a_1 \bigg| \frac{a_t}{t}, \frac{(1-t)^2}{t^2} I\right).
\end{equation}
By comparing the Gaussian normalizing constants, one has the exact density relation 
\begin{equation}
    p_t(a_t \mid a_1) = t^{-d} \, \phi_{1|t}(a_1 \mid a_t), \qquad d=\dim(\mathcal{A}). \label{eq:pt_phi_relation}
\end{equation}
The time-only factor $t^{-d}$ is independent of $\theta$ and can be absorbed into the time weighting / sampling distribution; we omit it for clarity (see Appendix~\ref{subsec:derivations_flame_r}). 

Instead of sampling $a_1$ from $\pi_{\text{new}}$, we first draw an intermediate action $a_t \sim h_t(\cdot\mid s)$, and then reverse-sample a candidate terminal action $a_1 \sim \phi_{1|t}(\cdot\mid a_t)$. 

The final tractable Q-Reweighted Flow Matching (QRFM) objective is:
\begin{equation}
\label{eq:qwfm_loss}
\begin{aligned}
    \mathcal{L}_{\text{QRFM}}(\theta) = \mathbb{E}\!_{\substack{t \sim \mathcal{U}[\varepsilon,1] \\ a_t \sim h_t(\cdot|s) \\ a_1 \sim \phi_{1|t}(\cdot|a_t)}} \bigg[ & \exp\left(\frac{Q(s,a_1)}{\alpha}\right) \\
    \times& \big\|u_{\theta}(a_t,t,s) - (a_1-a_0)\big\|^2 \bigg].
\end{aligned}
\end{equation}

where $a_0 = (a_t - t a_1)/(1-t)$ is deterministically recovered.

For the FLAME-M, we apply the same reweighting strategy to the variational MeanFlow objective (Eq.~\eqref{eq:meanflow_objective}), yielding the Q-Reweighted MeanFlow (QRMF) objective:
\begin{equation}
\label{eq:mpmd_loss}
\begin{aligned}
    \mathcal{L}_{\text{QRMF}}(\theta) = \mathbb{E}_{\substack{t \sim \mathcal{U}[\varepsilon,1] \\ a_t \sim h_t(\cdot\mid s) \\ a_1 \sim \phi_{1|t}(\cdot\mid a_t)}} \bigg[ & \exp\left(\frac{Q(s, a_1)}{\alpha}\right) \\
     \times& \left\| \overline{u}_{\theta}(a_t, \zeta, t, s) - \text{sg}(\overline{u}_{\text{tgt}}) \right\|^2 \bigg].
\end{aligned}
\end{equation}
where $\overline{u}_{\text{tgt}}$ is defined in Eq.~\eqref{eq:meanflow_target}. 
Detailed derivations for both FLAME-R and FLAME-M are provided in Appendix~\ref{subsec:derivations_flame_r} and Appendix~\ref{subsec:derivations_flame_m}, respectively.

\subsection{Entropy Regularization via Continuous Flows}
\label{subsec:log_likelihood}

To incorporate flow-based policies into the MaxEnt RL framework, the critic update requires the entropy term $\alpha \log \pi(a|s)$. Since $\pi_\theta$ is defined implicitly by a continuous flow, we compute $\log \pi(a_1|s)$ by integrating the negative divergence of the velocity field along the trajectory via the augmented dynamics~\cite{grathwohl2018}:
\begin{equation}
    \label{eq:aug_ode}
    \frac{d}{dt} \begin{bmatrix} a_t \\ \ell_t \end{bmatrix} = \begin{bmatrix} u_{\theta}(a_t, t, s) \\ -\text{Tr}\left( \frac{\partial u_{\theta}(a_t, t, s)}{\partial a_t} \right) \end{bmatrix}, \quad z(0) = \begin{bmatrix} a_0 \\ \log p_0(a_0) \end{bmatrix},
\end{equation}
where $\ell_t$ denotes the accumulated log-density change. Integrating from $t=0$ to $1$ yields the terminal log-likelihood $\log \pi(a_1|s)=\ell_1$. We employ Hutchinson's estimator to reduce the computational cost of the trace operator ($\text{Tr}$) from $O(d^2)$ to $O(d)$, where $d$ represents the action dimension.

For FLAME-M, a naive one-step log-likelihood approximation $\Delta \log p \approx -\text{Tr}(\nabla_a \overline{u}_\theta)$ ignores higher-order flow curvature terms, introducing significant discretization bias (as formalized in Proposition~\ref{prop:single_step_bias}). To reconcile efficient inference with accurate regularization, we propose a decoupled strategy where the actor and critic operate with distinct integration schemes. Specifically, the actor generates actions in a single step via $a_1 = a_0 + \overline{u}_\theta(a_0, 0, 1, s)$ to ensure minimal interaction latency. In contrast, the critic solves the augmented ODE (Eq.~\ref{eq:aug_ode}) by partitioning the interval $[0, 1]$ into $N_{\text{est}}$ sub-intervals. For each integration step $[t, t+\Delta t]$, the displacement and log-density change are determined by the local average velocity $\overline{u}_\theta(a_t, t, t+\Delta t, s)$. This approach restores the precision of entropy regularization for stable exploration while confining the computational overhead of multi-step integration strictly to the training phase.

\begin{proposition}[Discretization Error in Density Estimation]
\label{prop:single_step_bias}
Let the true log-density change be $\Delta \log p_{\text{true}}$ and the single-step estimator be $\Delta \log p_{\text{est}} = -\text{Tr}(J_{\overline{u}})$, where $J_{\overline{u}} = \nabla_a \overline{u}_\theta$. The discretization error $\mathcal{E}_{\text{single}}$ is dominated by the squared norm of the Jacobian \citep{benton2024error}:
\begin{equation}
    \mathcal{E}_{\text{single}} = \left| \Delta \log p_{\text{true}} - \Delta \log p_{\text{est}} \right| \approx \frac{1}{2} \text{Tr}(J_{\overline{u}}^2).
\end{equation}
\end{proposition}

This proposition implies a hard error bound for one-step estimation scaling with $O(\|J_{\overline{u}}\|^2)$. To resolve this for FLAME-M, a decoupled strategy is adopted: actions are generated in a single step ($N_{\text{gen}}=1$) for efficient interaction, while the log-likelihood is computed by solving the augmented ODE for the average velocity field $\overline{u}_\theta$ using a finer grid ($N_{\text{est}}=5$ in all MuJoCo experiments) during the critic update. This suppresses the error as derived below:

\begin{corollary}[Multi-Step Error Suppression]
\label{cor:multi_step_error}
By discretizing the integration path of the average velocity field into $N_{\text{est}}$ substeps, the cumulative log-likelihood error $\mathcal{E}_{\text{multi}}$ is reduced linearly with respect to $N_{\text{est}}$:
\begin{equation}
    \mathcal{E}_{\text{multi}} \approx \frac{1}{2N_{\text{est}}}\text{Tr}(J_{\overline{u}}^2).
\end{equation}
\end{corollary}

This strategy ensures that the entropy estimate used for the critic update is sufficiently accurate to maintain stable training, while keeping environment interaction strictly one-step (NFE=1) since multi-step integration is performed only in the critic update.

\subsection{Practical Implementation and Algorithms}
\label{subsec:practical_and_algos}

Translating the theoretical framework of FLAME into a stable online algorithm requires addressing numerical instability in importance sampling and enforcing physical constraints on generated actions.

\textbf{Numerical Stability via LogSumExp.}
The Q-weighted objectives in Eq.~\eqref{eq:qwfm_loss} and Eq.~\eqref{eq:mpmd_loss} involve importance weights proportional to $\exp(Q(s,a)/\alpha)$, which may cause numerical overflow when Q-values have large magnitudes.
To mitigate this, we employ the LogSumExp trick to compute normalized importance weights (equivalent to a softmax distribution).
For a batch of $K$ candidates $\{\hat a_1^i\}_{i=1}^K$, let $q_i = Q(s,\hat a_1^i)/\alpha$ and $\mathrm{LSE}(\mathbf{q}) = \log\sum_{j=1}^K \exp(q_j)$. We compute the weights as:
\begin{equation}
    w_i = \exp\left( q_i - \mathrm{LSE}(\mathbf{q}) \right).
\end{equation}
This formulation guarantees that $\sum_{i=1}^K w_i = 1$ and $w_i \in (0,1]$, effectively implementing Self-Normalized Importance Sampling (SNIS).
Crucially, this normalization prevents numerical instability caused by exploding exponential terms and stabilizes the gradient magnitude across batches, ensuring robust optimization.

\textbf{Boundary Constrained Reverse Sampling.} 
To respect the bounded action space and prevent value overestimation from out-of-distribution queries, we sample the noise for the reverse kernel from a truncated normal distribution. The truncation limits are derived from the environment's action boundaries, ensuring that the reconstructed candidates $\hat{a}_1$ fall strictly within the valid support. This constraint prevents querying the critic on undefined actions, thereby stabilizing the policy improvement step.

The unified training procedure for both FLAME-R and FLAME-M is summarized in Algorithm \ref{alg:flame_unified}.
\begin{algorithm}[tb]
\caption{FLAME: Q-Reweighted Flow Matching for MaxEnt RL}
\label{alg:flame_unified}
\begin{algorithmic}[1]
\REQUIRE Initial policy $\pi_\theta$, Q-function $Q_\phi$, Temperature $\alpha$, Replay buffer $\mathcal{D} = \emptyset$.
\FOR{iteration $k = 1, 2, \dots$}
    \STATE Interact with env using $\pi_\theta$ (ODE for FLAME-R, One-step for FLAME-M) and update $\mathcal{D}$
    \STATE Sample batch $\mathcal{B} = \{(s, a, r, s')\} \sim \mathcal{D}$
    \STATE Sample next action $a'$ from $\pi_\theta(\cdot|s')$ (Flow Sampling)
    \STATE \textbf{Critic learning:} Estimate $\log \pi_\theta(a'|s')$ (Unbiased for R; Decoupled for M) and update $Q_\phi$
    \STATE \textbf{Actor learning:} Update $\pi_\theta$ via Q-reweighted Flow Matching:
    \STATE \hspace{0.5em} $\left\{
    \begin{aligned}
    \textbf{FLAME-R: } & \text{Sample } t \sim \mathcal{U}, a_t, \text{truncated noise } \epsilon. \\
    & \text{Update } \theta \text{ by } \mathcal{L}_{\text{QRFM}} \text{ (Eq.~\ref{eq:qwfm_loss})} \\
    \textbf{FLAME-M: } & \text{Sample } t \sim \mathcal{U}, a_t, \text{truncated noise } \epsilon. \\
    & \text{Update } \theta \text{ by } \mathcal{L}_{\text{QRMF}} \text{ (Eq.~\ref{eq:mpmd_loss})}
    \end{aligned}
    \right.$
    \STATE \textbf{Temperature:} Update $\alpha$ towards target entropy $\mathcal{H}_0$
\ENDFOR
\end{algorithmic}
\end{algorithm}

\begin{table*}[!t]
\centering
\caption{Performance on OpenAI Gym MuJoCo environments (v5). We compare FLAME against classic RL, diffusion policies, and other flow-based methods. Results are mean $\pm$ std over 5 seeds. \colorbox{lightblue}{Blue} highlights the best overall performance. 32 denotes the number of candidate actions mentioned in Section \ref{app:implementation_details}.}
\label{tab:full_results_stacked}
\tiny
\setlength{\tabcolsep}{4pt}

\begin{tabular}{llrrrrr}
\toprule
& & \textsc{HalfCheetah} & \textsc{Reacher} & \textsc{Humanoid} & \textsc{Pusher} & \textsc{InvPendulum} \\
\midrule

\multirow{3}{*}{\begin{tabular}{@{}c@{}}Classic\\ Model-Free RL\end{tabular}} 
& PPO (NFE=1) & $4631.15 \pm 345.60$ & $-6.85 \pm 0.12$ & $1023.24 \pm 155.20$ & $-28.12 \pm 1.50$ & \cellcolor{lightblue}$1000.00 \pm 0.00$ \\
& TD3 (NFE=1) & $7942.14 \pm 569.13$ & $-3.55 \pm 0.16$ & $5824.23 \pm 257.20$ & $-26.17 \pm 2.13$ & \cellcolor{lightblue}$1000.00 \pm 0.00$ \\
& SAC (NFE=1) & $11379.74 \pm 474.51$ & $-67.29 \pm 0.05$ & $4623.49 \pm 48.50$ & $-30.52 \pm 0.19$ & \cellcolor{lightblue}$1000.00 \pm 0.00$ \\
\midrule

\multirow{6}{*}{\begin{tabular}{@{}c@{}}Diffusion\\ Policy RL\end{tabular}} 
& DIPO (NFE=20) & $9105.59 \pm 6.06$ & $-3.57 \pm 0.08$ & $5311.28 \pm 73.65$ & $-34.40 \pm 1.47$ & \cellcolor{lightblue}$1000.00 \pm 0.00$ \\
& DACER (NFE=20) & $11089.40 \pm 86.40$ & $-3.76 \pm 0.02$ & $3388.63 \pm 2118.66$ & $-30.91 \pm 0.44$ & \cellcolor{lightblue}$1000.00 \pm 0.00$ \\
& QSM (NFE=20$\times$32) & $10227.18 \pm 688.22$ & $-4.19 \pm 0.02$ & $5200.17 \pm 186.66$ & $-75.01 \pm 0.43$ & $810.07 \pm 122.98$ \\
& QVPO (NFE=20$\times$32) & $7755.91 \pm 1645.55$ & $-30.70 \pm 34.21$ & $428.24 \pm 70.85$ & $-130.43 \pm 0.57$ & \cellcolor{lightblue}$1000.00 \pm 0.00$ \\
& SDAC (NFE=20$\times$32) & $11165.70 \pm 164.25$ & $-7.59 \pm 0.07$ & $5590.20 \pm 480.51$ & $-38.10 \pm 0.15$ & \cellcolor{lightblue}$1000.00 \pm 0.00$ \\
& DPMD (NFE=20$\times$32) & $11041.47 \pm 557.77$ & $-3.18 \pm 0.06$ & $6323.65 \pm 210.26$ & $-30.71 \pm 1.06$ & \cellcolor{lightblue}$1000.00 \pm 0.00$ \\
\midrule

\multirow{6}{*}{\begin{tabular}{@{}c@{}}Flow\\ Policy RL\end{tabular}} 
& FPMD-R (NFE=1) & $9766.20 \pm 103.40$ & $-3.30 \pm 0.16$ & $6010.18 \pm 138.17$ & $-37.36 \pm 2.48$ & \cellcolor{lightblue}$1000.00 \pm 0.00$ \\
& FPMD-M (NFE=1) & $9156.31 \pm 294.50$ & $-3.37 \pm 0.03$ & $5893.78 \pm 310.06$ & $-27.39 \pm 1.11$ & \cellcolor{lightblue}$1000.00 \pm 0.00$ \\
& FlowRL (NFE=1) & $6977.13 \pm 1034.07$ & $-3.94 \pm 0.74$ & $5159.71 \pm 263.41$ & $-38.04 \pm 9.81$ & \cellcolor{lightblue}$1000.00 \pm 0.00$ \\
& MEOW (NFE=1)& $10132.20 \pm 1846.90$ & $-4.40 \pm 0.10$ & $5479.30 \pm 454.70$ & $-28.70 \pm 1.60$ & \cellcolor{lightblue}$1000.00 \pm 0.00$ \\
\cmidrule{2-7} 
& \textbf{FLAME-R (Ours)} & \cellcolor{lightblue}$11945.08 \pm 461.47$ & \cellcolor{lightblue}$-3.16 \pm 0.13$ & \cellcolor{lightblue}$6987.11 \pm 64.71$ & $-30.11 \pm 0.66$ & \cellcolor{lightblue}$1000.00 \pm 0.00$ \\
& \textbf{FLAME-M (Ours)} & $10600.36 \pm 1003.33$ & $-3.54 \pm 0.01$ & $6301.28 \pm 105.74$ & \cellcolor{lightblue}$-22.70 \pm 1.25$ & \cellcolor{lightblue}$1000.00 \pm 0.00$ \\
\bottomrule
\toprule 

& & \textsc{Ant} & \textsc{Hopper} & \textsc{Swimmer} & \textsc{Walker2d} & \textsc{InvDoublePendulum} \\
\midrule

\multirow{3}{*}{\begin{tabular}{@{}c@{}}Classic\\ Model-Free RL\end{tabular}} 
& PPO (NFE=1) & $3442.50 \pm 851.00$ & $3227.40 \pm 164.00$ & $84.50 \pm 12.40$ & $4114.20 \pm 806.00$ & $9358.00 \pm 1.00$ \\
& TD3 (NFE=1) & $3733.60 \pm 133.60$ & $1934.10 \pm 107.90$ & $71.90 \pm 1.50$ & $2476.50 \pm 135.70$ & \cellcolor{lightblue}$9360.00 \pm 0.00$ \\
& SAC (NFE=1) & $4615.28 \pm 16.56$ & $2525.27 \pm 1016.81$ & $105.13 \pm 5.43$ & $3892.49 \pm 1990.29$ & $9023.00 \pm 14.00$ \\
\midrule

\multirow{6}{*}{\begin{tabular}{@{}c@{}}Diffusion\\ Policy RL\end{tabular}} 
& DIPO (NFE=20) & $1291.87 \pm 711.75$ & $731.02 \pm 613.29$ & $46.76 \pm 2.76$ & $2726.00 \pm 1060.33$ & $9051.55 \pm 37.55$ \\
& DACER (NFE=20) & $4307.53 \pm 269.60$ & $3138.25 \pm 485.55$ & $98.66 \pm 51.65$ & $3135.81 \pm 2186.75$ & $6294.69 \pm 70.81$ \\
& QSM (NFE=20$\times$32) & $542.52 \pm 125.39$ & $3192.68 \pm 665.71$ & $59.95 \pm 28.07$ & $1785.40 \pm 986.17$ & $2472.77 \pm 22.96$ \\
& QVPO (NFE=20$\times$32) & $2145.44 \pm 1972.56$ & $2150.44 \pm 11.23$ & $48.74 \pm 5.17$ & $350.89 \pm 496.66$ & $9350.00 \pm 5.00$ \\
& SDAC (NFE=20$\times$32) & $4528.43 \pm 98.39$ & $1438.44 \pm 771.43$ & $59.09 \pm 0.32$ & $4195.03 \pm 85.31$ & $9173.00 \pm 57.00$ \\
& DPMD (NFE=20$\times$32) & $4964.22 \pm 282.86$ & $3229.02 \pm 4.53$ & $76.01 \pm 0.10$ & $4364.94 \pm 227.55$ & \cellcolor{lightblue}$9360.00 \pm 0.00$ \\
\midrule

\multirow{6}{*}{\begin{tabular}{@{}c@{}}Flow\\ Policy RL\end{tabular}} 
& FPMD-R (NFE=1) & $5332.70 \pm 162.66$ & $3185.47 \pm 3.61$ & $60.92 \pm 0.38$ & $3389.89 \pm 237.25$ & $9354.00 \pm 1.44$ \\
& FPMD-M (NFE=1) & $5429.42 \pm 606.87$ & $2046.67 \pm 491.84$ & $55.46 \pm 2.67$ & $4044.61 \pm 321.25$ & $9351.00 \pm 4.00$ \\
& FlowRL (NFE=1) & $5498.69 \pm 945.36$ & $1922.11 \pm 734.83$ & $90.15 \pm 17.88$ & $4536.44 \pm 533.63$ & $9207.51 \pm 93.36$ \\
& MEOW (NFE=1)& $5608.30 \pm 233.60$ & $2918.60 \pm 154.50$ & $41.20 \pm 1.80$ & $4355.10 \pm 55.50$ & $8840.00 \pm 9.00$ \\
\cmidrule{2-7} 
& \textbf{FLAME-R (Ours)} & \cellcolor{lightblue}$6121.97 \pm 245.09$ & \cellcolor{lightblue}$3327.34 \pm 179.38$ & \cellcolor{lightblue}$146.11 \pm 9.28$ & $3154.73 \pm 944.73$ & $9359.00 \pm 1.00$ \\
& \textbf{FLAME-M (Ours)} & $5191.92 \pm 501.00$ & $2074.26 \pm 1710.42$ & $57.61 \pm 1.34$ & \cellcolor{lightblue}$5236.50 \pm 317.45$ & \cellcolor{lightblue}$9360.00 \pm 0.00$ \\
\bottomrule
\end{tabular}
\end{table*}

\section{Experiments}
\label{sec:experiments}

We evaluate FLAME on standard continuous control benchmarks to validate its effectiveness, focusing on three primary objectives: demonstrating that FLAME achieves state-of-the-art performance with one-step inference comparable to multi-step diffusion models; verifying the necessity of the proposed decoupled entropy estimation (FLAME-M) and ODE integration (FLAME-R) for stability; and showcasing its capability to capture complex multi-modal distributions beyond Gaussian limitations. These aspects are addressed in Secs.~\ref{subsec:multimodal_vis}--\ref{subsec:efficiency} and Appendix~\ref{app:ablation}.

\subsection{Experimental Setup}
We evaluate our method on 10 standard MuJoCo-v5 continuous control benchmarks, ranging from tasks with specific dynamics like \textsc{InvertedPendulum} to complex high-dimensional locomotion tasks such as \textsc{Humanoid}, \textsc{Ant}, and \textsc{Walker2d}. Implemented within the JAX framework, our models are trained for 200K iterations (1 million interactions) on most tasks, with the training budget doubled to 400K iterations (2 million interactions) for the more challenging \textsc{Humanoid} environment. Specific implementation details and hyperparameters for FLAME are provided in Appendix \ref{app:implementation_details} and \ref{app:Hyperparameterssettings}, respectively.

We compare FLAME against a comprehensive set of baselines spanning three categories: classic model-free algorithms with single-step inference (SAC~\citep{haarnoja2019softactorcriticalgorithmsapplications}, PPO~\citep{schulman2017proximalpolicyoptimizationalgorithms}, TD3~\citep{pmlr-v80-fujimoto18a}), diffusion-based methods that typically require expensive iterative sampling (DIPO~\citep{yang2023policyrepresentationdiffusionprobability}, DACER~\citep{wang2024diffusion}, QSM~\citep{psenka2025learningdiffusionmodelpolicy}, QVPO~\citep{ding2024diffusion}, SDAC~\citep{ma2025effi}, DPMD~\citep{ma2025effi}), and recent flow-based approaches targeting efficient inference (FPMD~\citep{chen2025onestepflowpolicymirror}, FlowRL~\citep{lv2025flowbasedpolicyonlinereinforcement}, MEOW~\citep{chao2024}). Further details are provided in Appendix \ref{app:baseline_details}. Unless otherwise specified, we use $N_{\text{est}}=5$ integration steps for log-likelihood estimation across all MuJoCo tasks. A more comprehensive sensitivity analysis of key hyperparameters (e.g., $K$, $N_{\text{est}}$) is provided in Appendix~\ref{app:ablation}.

\begin{figure*}[t]
    \centering
    \setlength{\tabcolsep}{1pt} 
    \renewcommand{\arraystretch}{0.5} 
    
    \newcommand{\incplot}[1]{\includegraphics[width=0.163\linewidth]{#1}}

    \begin{tabular}{cccccc}
        \incplot{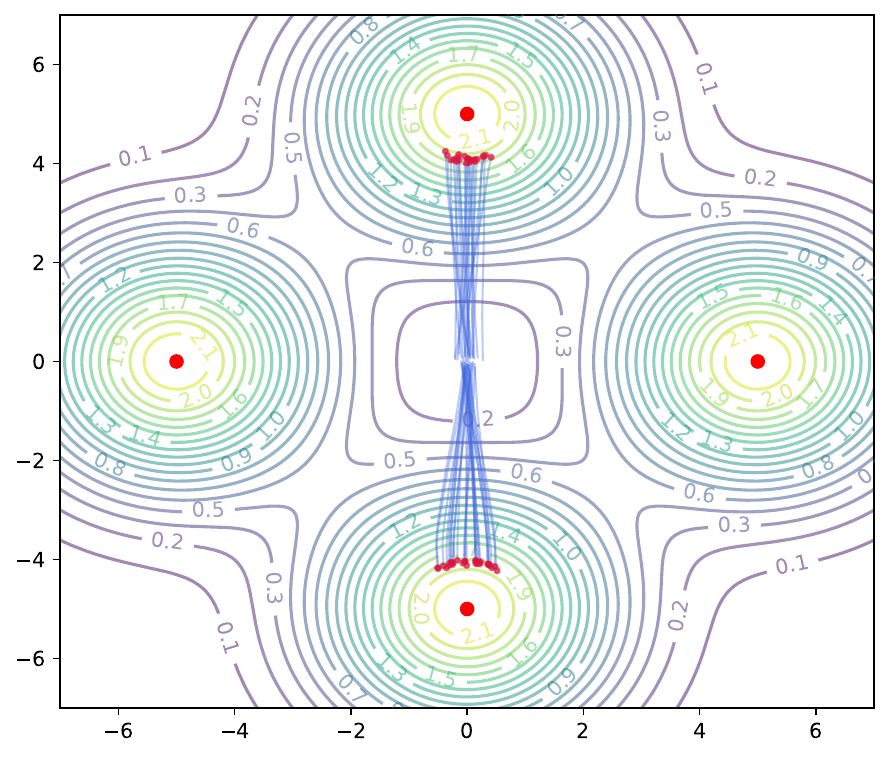} & 
        \incplot{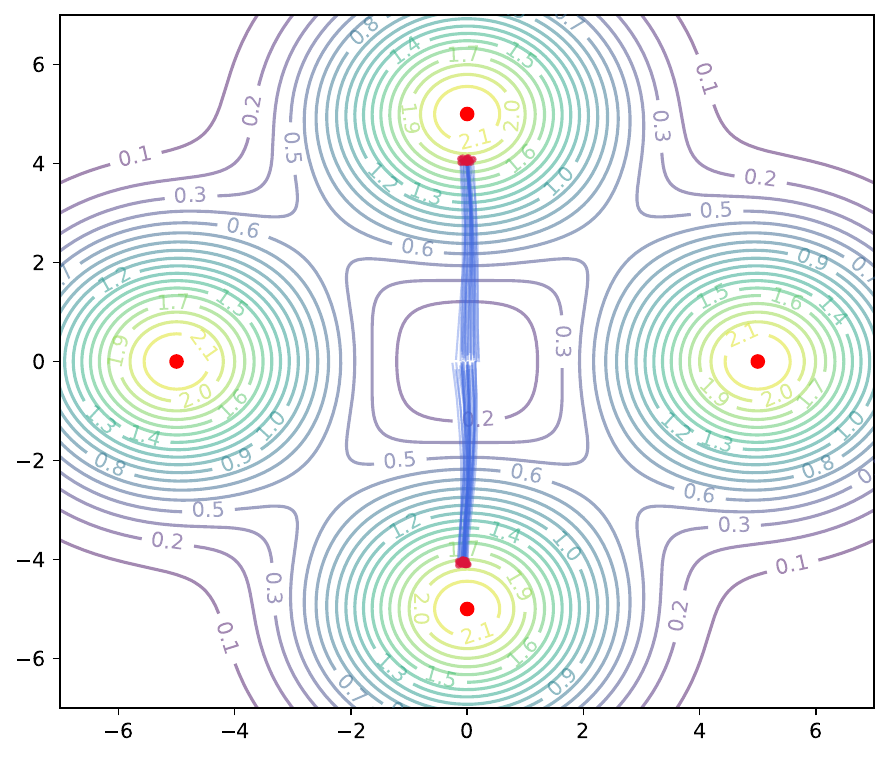} & 
        \incplot{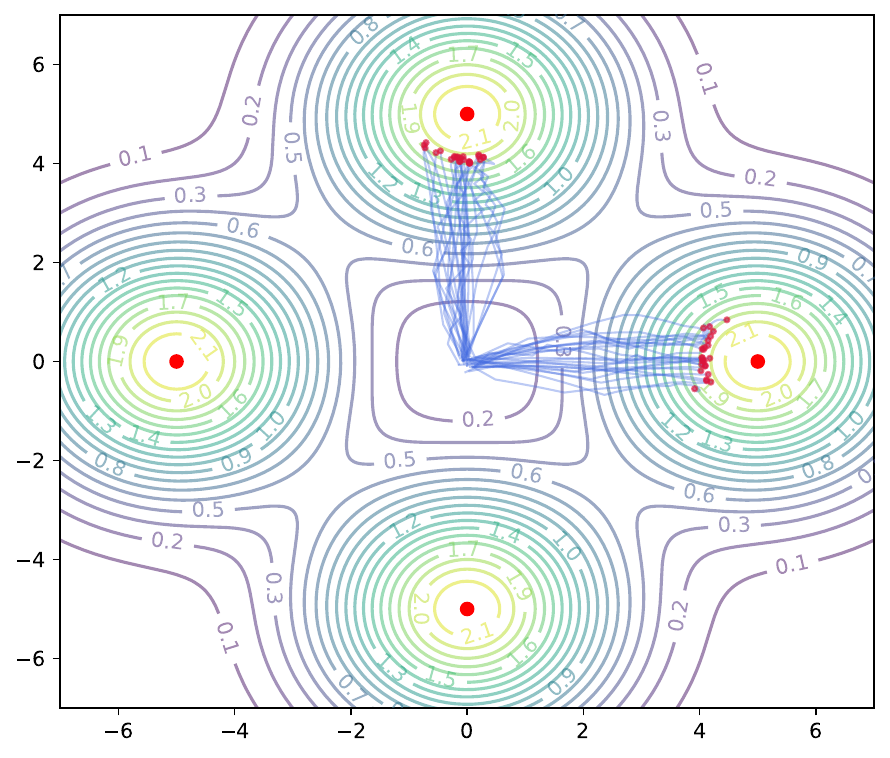} & 
        \incplot{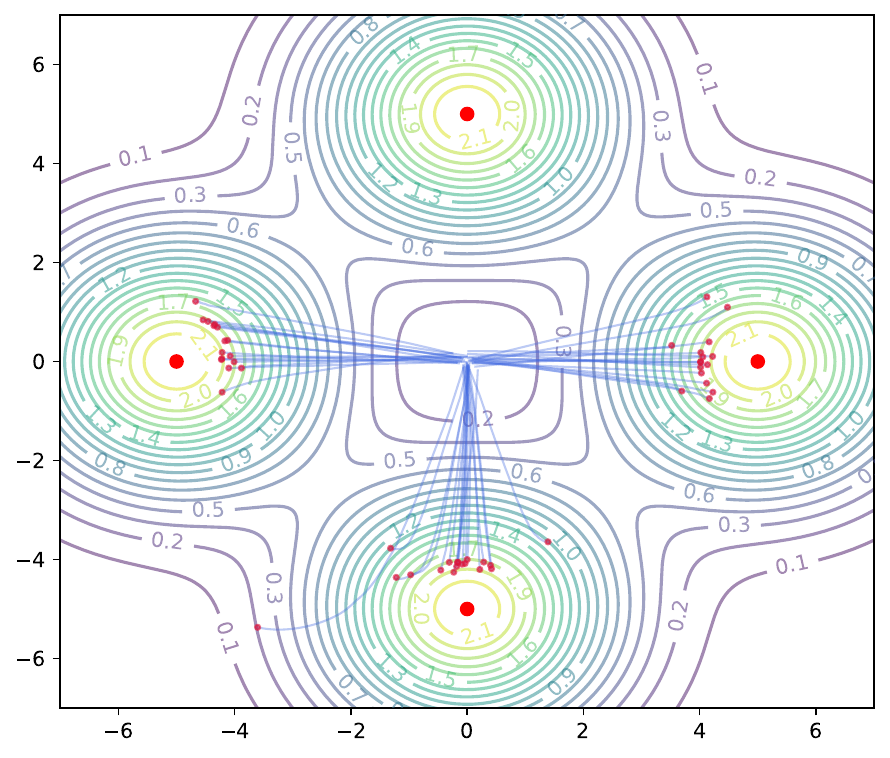} & 
        \incplot{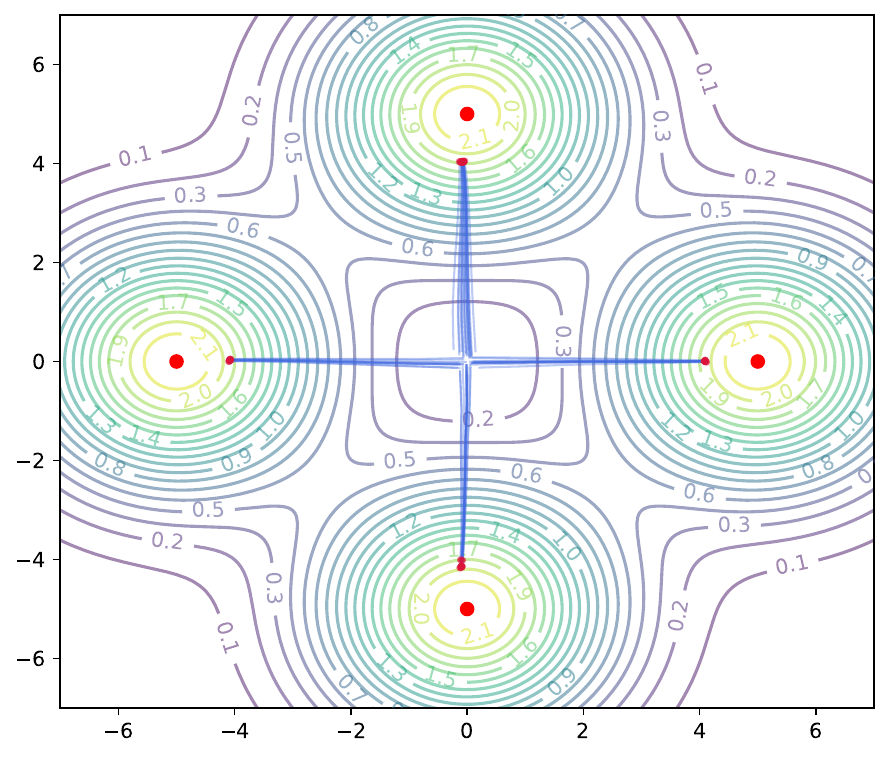} &
        \incplot{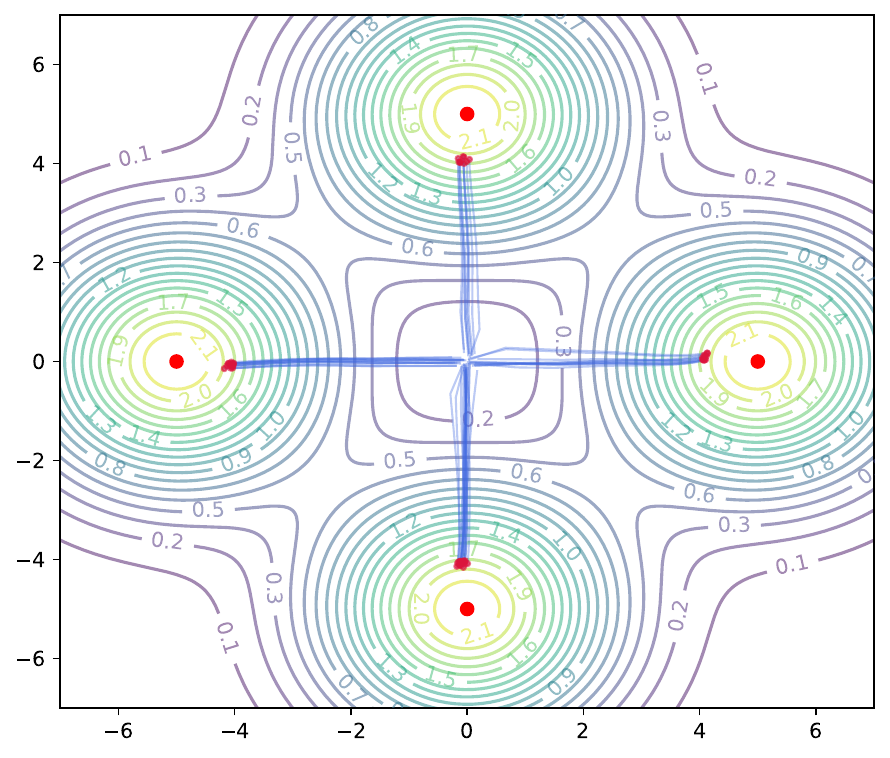} \\
        
        \incplot{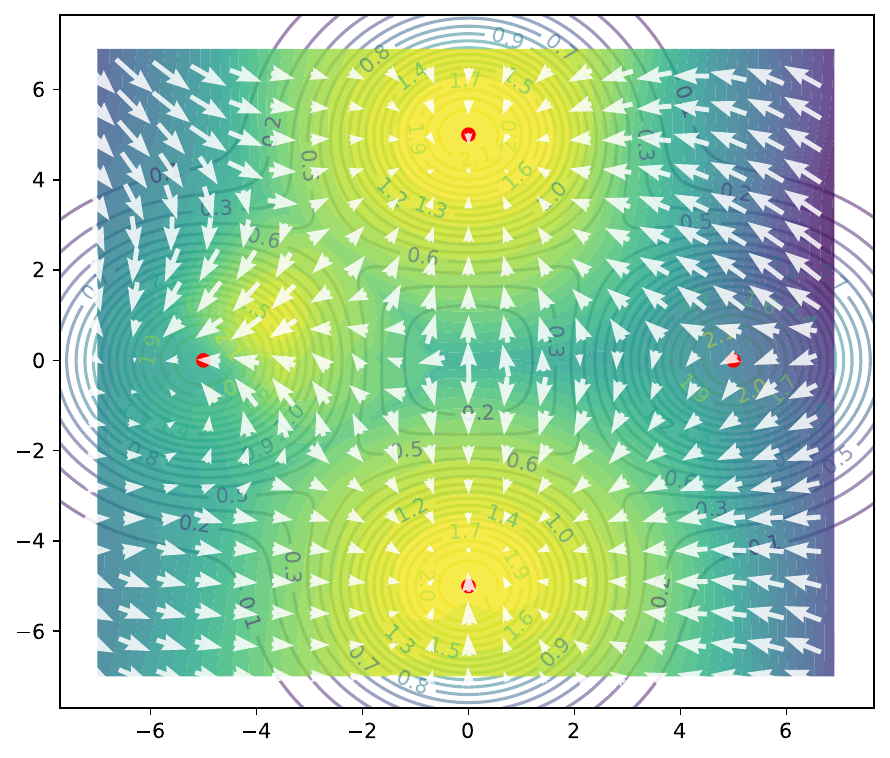} & 
        \incplot{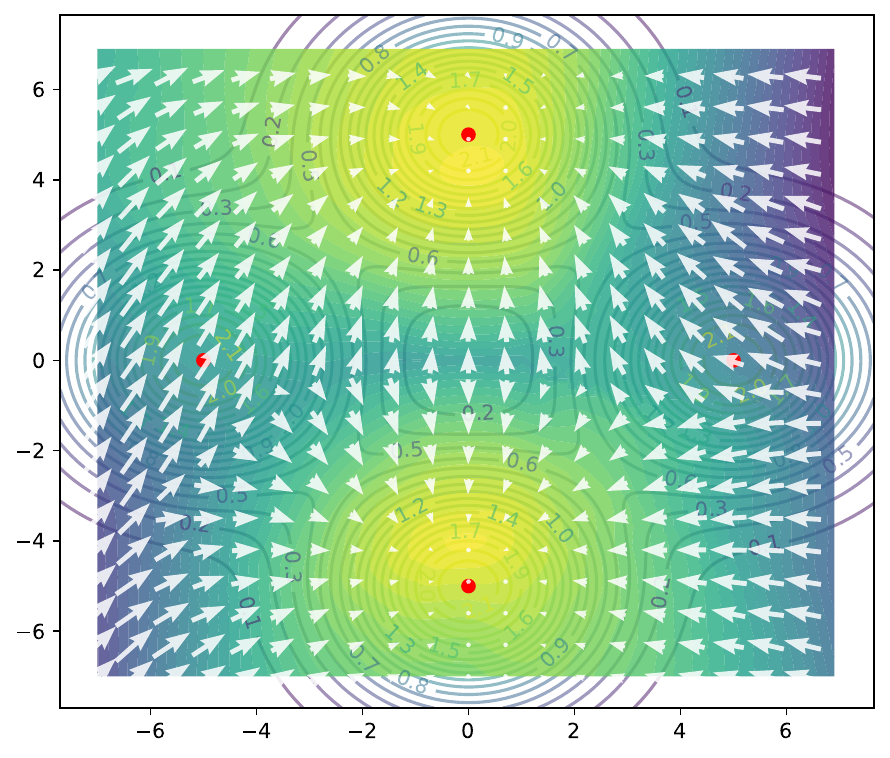} & 
        \incplot{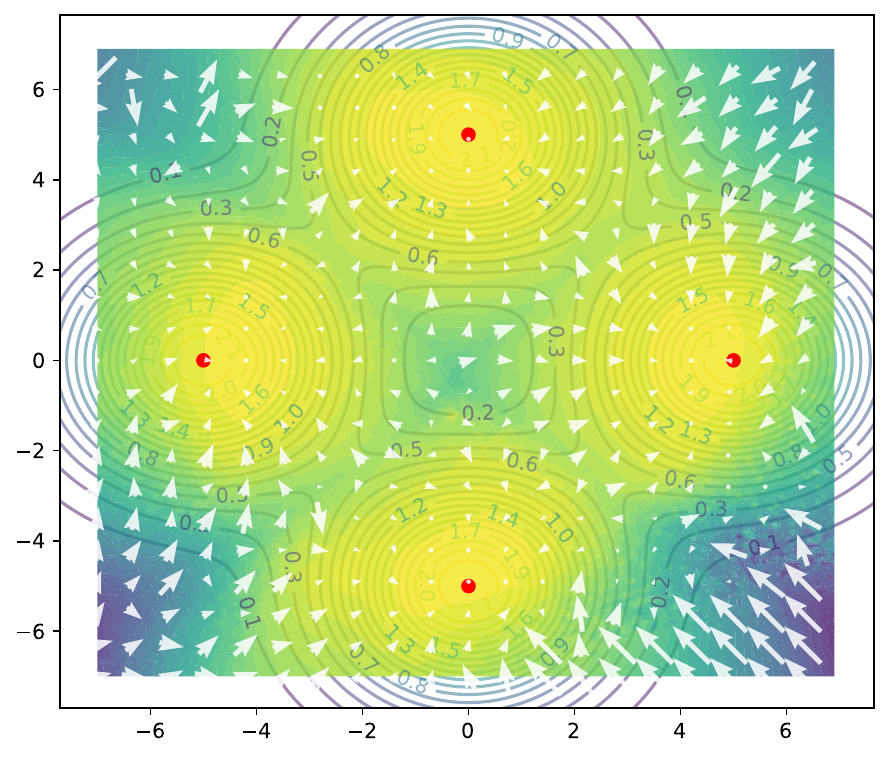} & 
        \incplot{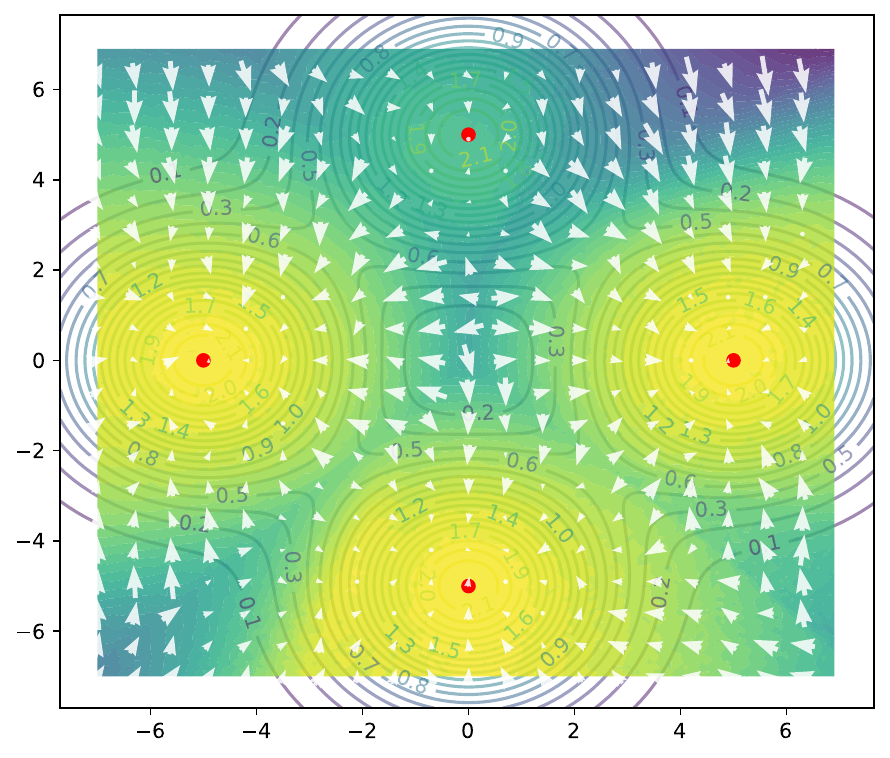} & 
        \incplot{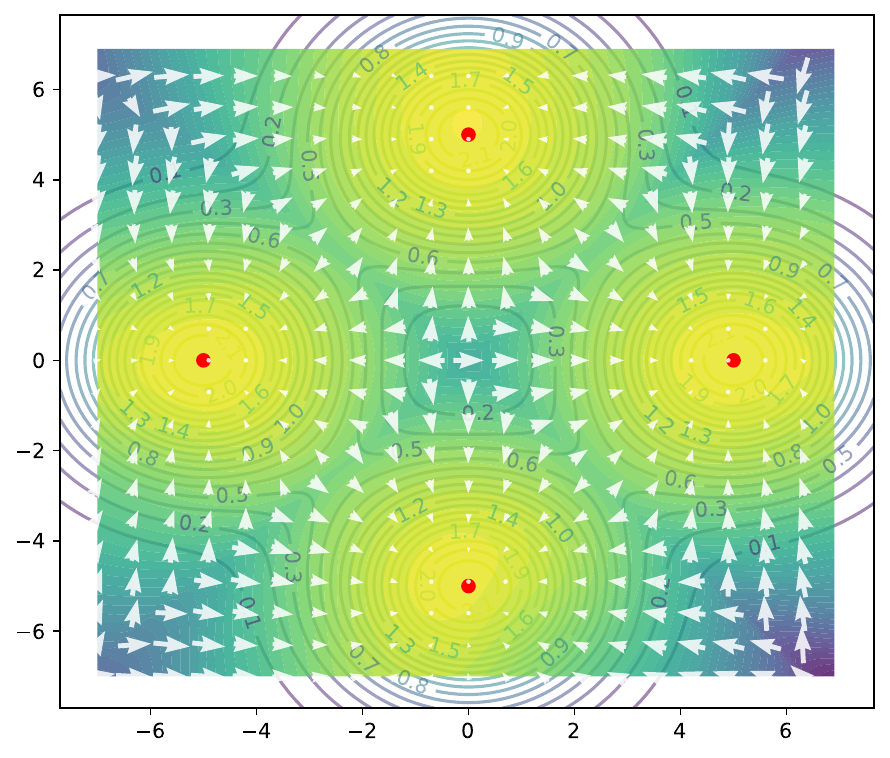} &
        \incplot{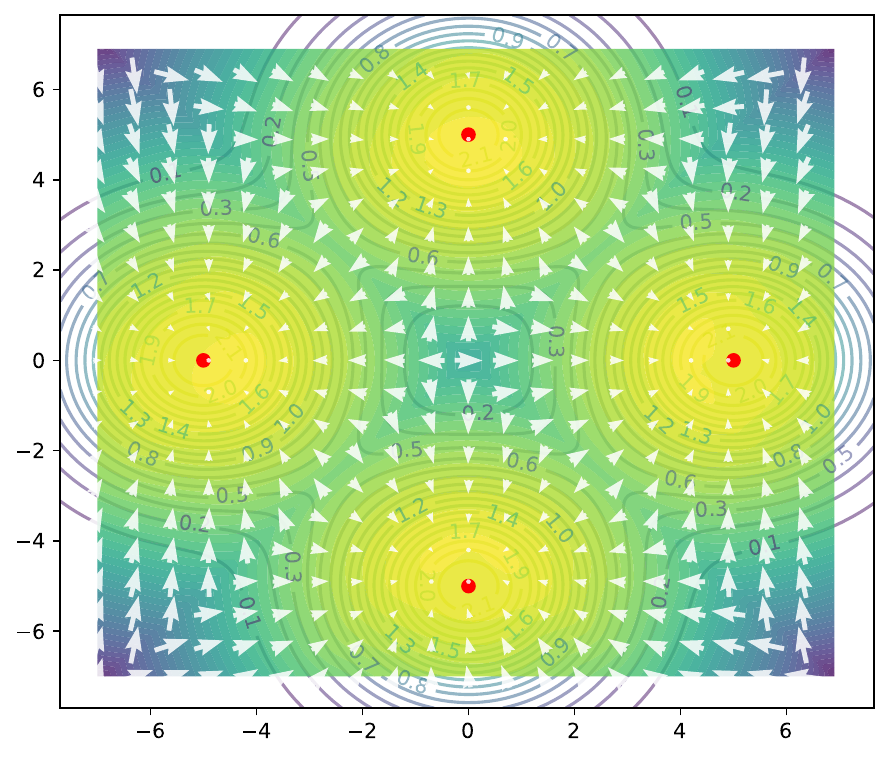} \\
        
        \scriptsize (a) SAC & 
        \scriptsize (b) SDAC & 
        \scriptsize (c) FPMD-R & 
        \scriptsize (d) FLAME (w/o Ent) & 
        \scriptsize (e) \textbf{FLAME-R} &
        \scriptsize (f) \textbf{FLAME-M} \\
    \end{tabular}
    
\caption{\textbf{MultiGoal multimodality.} Gaussian policies (SAC) and methods without entropy regulation collapse to one goal, while \textbf{FLAME-R/M} cover all four symmetric goals, demonstrating stable MaxEnt exploration.}

    \label{fig:multimodal_vis}
\end{figure*}

\begin{figure}[t]
    \centering
    \begin{subfigure}{0.47\linewidth}
        \centering
        \includegraphics[width=\linewidth]{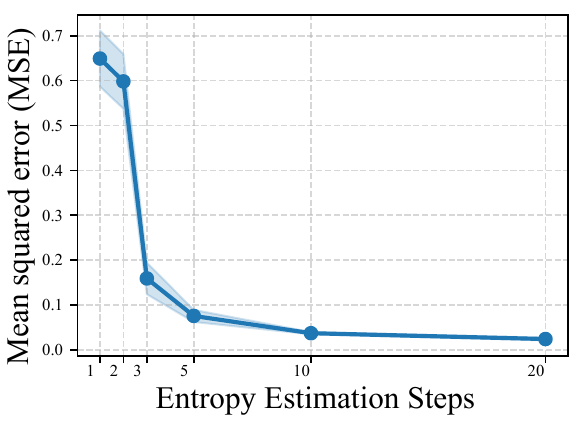} 
        \caption{FLAME-R MSE Curve}
    \end{subfigure}
    \hfill
    \begin{subfigure}{0.50\linewidth}
        \centering
        \includegraphics[width=\linewidth]{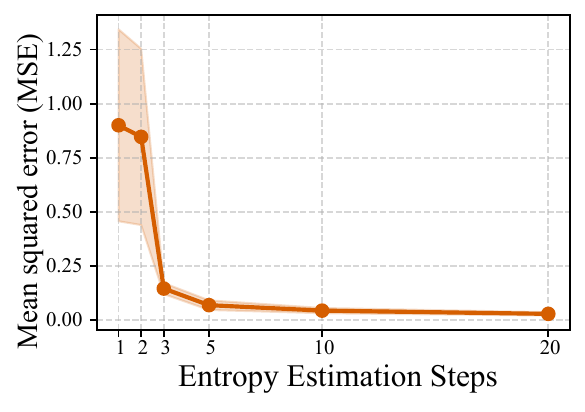} 
        \caption{FLAME-M MSE Curve}
    \end{subfigure}
    \vspace{-0.5em}
    \caption{\textbf{Log-Likelihood Estimation Error vs. Integration Steps ($N_{\text{est}}$).} 
    The Mean Squared Error (MSE) drops significantly as $N_{\text{est}}$ increases. The convergence at $N_{\text{est}}=5$ validates our choice of multi-step estimation for the critic to reduce bias.}
    \label{fig:entropy_mse_curves}
    \vspace{-1.5em} 
\end{figure}

\begin{figure*}[t]
    \centering
    \setlength{\tabcolsep}{1pt} 
    \renewcommand{\arraystretch}{0.5} 

    \begin{tabular}{cccccc}
        \includegraphics[width=0.16\linewidth]{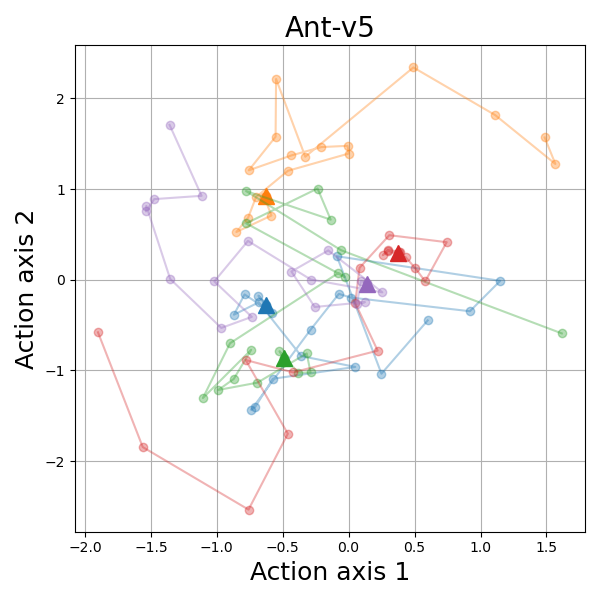} & 
        \includegraphics[width=0.16\linewidth]{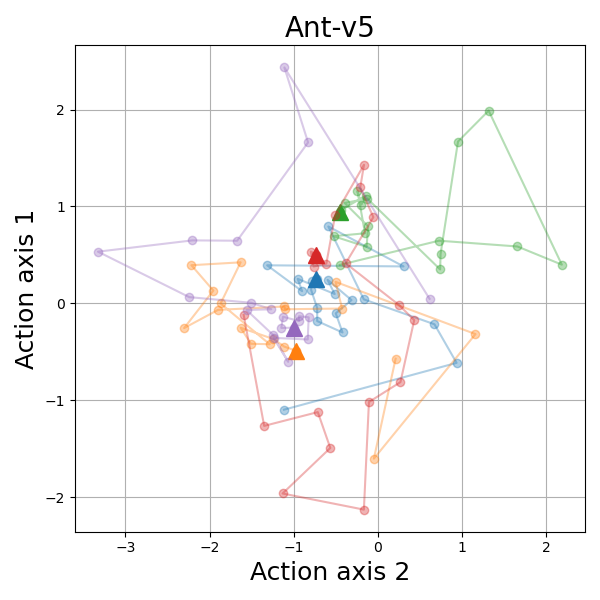} &
        \includegraphics[width=0.16\linewidth]{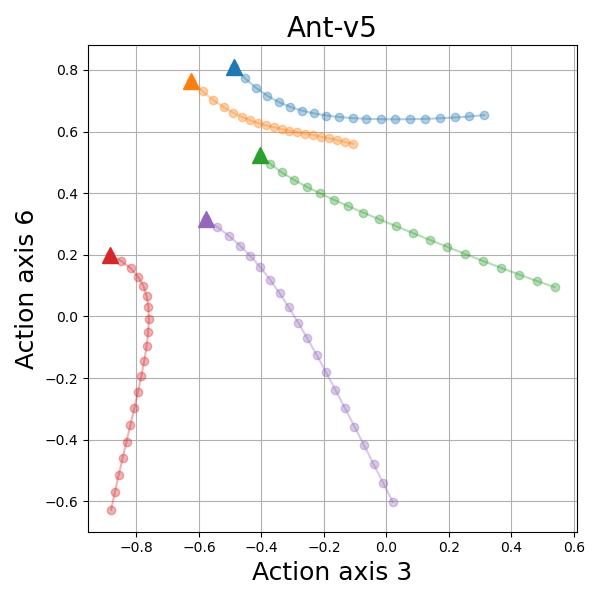} & 
        \includegraphics[width=0.16\linewidth]{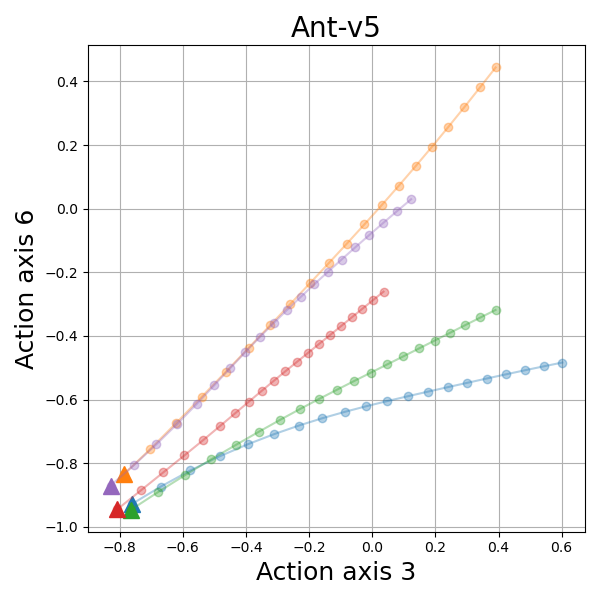} &
        \includegraphics[width=0.16\linewidth]{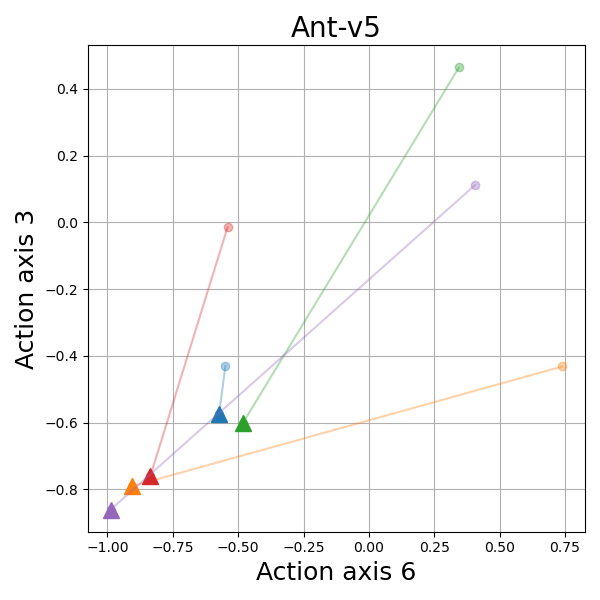} & 
        \includegraphics[width=0.16\linewidth]{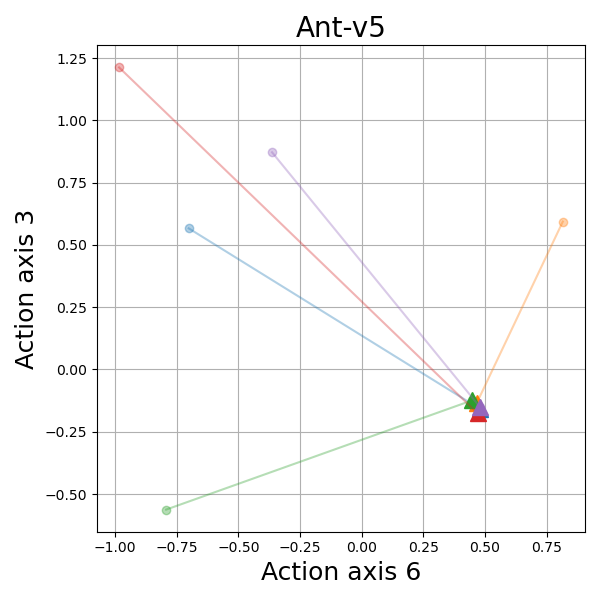} \\
        
        \scriptsize SDAC (5K) & \scriptsize SDAC (200K) & 
        \scriptsize FLAME-R (5K) & \scriptsize FLAME-R (200K) & 
        \scriptsize FLAME-M (5K) & \scriptsize FLAME-M (200K) \\
    \end{tabular}
    
    \caption{\textbf{Evolution of sampling trajectories.} 
    From left to right: Policies at 5K (early) and 200K (converged) iterations.
    \textbf{SDAC} retains curvature, requiring multi-step sampling. 
    \textbf{FLAME-R} gradually straightens the flow, enabling one-step inference at convergence. 
    \textbf{FLAME-M} achieves straight trajectories almost immediately (5K), demonstrating superior training efficiency.}
    \label{fig:traj_evolution}
\end{figure*}

\subsection{Performance on MuJoCo Benchmarks}

Table~\ref{tab:full_results_stacked} shows that FLAME effectively bridges the gap between expressivity and efficiency. In complex high-dimensional environments like \textsc{Ant} and \textsc{Humanoid}, FLAME-R consistently surpasses standard Gaussian policies, highlighting the advantage of flow-based representations in capturing complex dynamics that limit unimodal distributions. The full training curves are presented in Appendix~\ref{app:training_curves}.

Crucially, FLAME achieves these results with a single inference step, matching or exceeding the performance of multi-step diffusion baselines. Notably, on \textsc{HalfCheetah}, FLAME-R surpasses both the strong diffusion baseline SDAC and DPMD. Similarly, the FLAME-M demonstrates robust performance on \textsc{Walker2d}, outperforming comparable methods. These results indicate that the proposed Q-reweighted flow matching objective effectively aligns the policy distribution with high-value regions while enabling efficient one-step generation.

Furthermore, FLAME consistently outperforms other flow-based baselines. On challenging tasks such as \textsc{Humanoid}, FLAME-R leads FlowRL by a substantial margin, underscoring the benefit of rigorously optimizing the MaxEnt objective. Unlike approaches relying on implicit or approximate entropy handling, FLAME explicitly accounts for log-likelihood estimation, enabling a more principled exploration-exploitation trade-off.

FLAME-R and FLAME-M offer complementary design trade-offs between theoretical rigor and deployment efficiency. FLAME-R prioritizes stability by employing unbiased entropy estimation via standard augmented ODE integration, yielding consistent low-variance performance. In contrast, FLAME-M adopts the decoupled strategy to mitigate discretization bias while strictly maintaining one-step inference. As shown in Table~\ref{tab:full_results_stacked}, FLAME-M remains highly competitive, even surpassing FLAME-R on tasks like \textsc{Walker2d} and \textsc{Pusher}, demonstrating that the decoupled approximation effectively balances regularization accuracy with minimal latency. In practice, FLAME-R serves as a robust default, while FLAME-M is ideal for latency-constrained applications. 

We also evaluate FLAME on visual RL tasks (Appendix~\ref{app:visual_rl}) to demonstrate its scalability to high-dimensional observations.

\subsection{Bias-Variance Analysis of Entropy Estimation}
\label{subsec:entropy_analysis}

To quantify the discretization bias discussed in Proposition \ref{prop:single_step_bias}, we conducted a controlled experiment on a 2D Gaussian Mixture Model (GMM). We trained both FLAME-R and FLAME-M policies to match a 4-mode target distribution and evaluated the Mean Squared Error (MSE) between the estimated log-likelihood and the ground truth density across varying integration steps $N_{\text{est}}$.

Figure \ref{fig:entropy_mse_curves} presents the quantitative results. For both FLAME-R and FLAME-M, the naive single-step estimator ($N_{\text{est}}=1$) yields a high MSE, indicating a significant deviation from the true entropy. This confirms that while one-step generation is feasible for actions, it is insufficient for the precise density estimation required by the critic. However, the error decays rapidly. At $N_{\text{est}}=5$, the MSE effectively converges, offering a rigorous approximation with minimal computational overhead. Consequently, we adopt $N_{\text{est}}=5$ for the critic's entropy estimation in our main algorithms, balancing numerical precision with training efficiency. Detailed visualizations of the density reconstruction at different steps are provided in Appendix \ref{app:toy_details}.

\subsection{Policy Representation on Multimodal Environment}
\label{subsec:multimodal_vis}
To investigate multimodal representation, we train all methods from scratch on a MultiGoal environment with four symmetric goals (Figure \ref{fig:multimodal_vis}). Gaussian policies and generative policies without entropy regulation (SDAC, FLAME w/o Ent) exhibit mode collapse, converging to a single goal, whereas both FLAME-R and FLAME-M cover all four modes. This supports RQ3 and indicates that accurate entropy estimation is necessary for stable exploration under the MaxEnt objective.

\subsection{Visual Analysis of Training Evolution}
\label{subsec:traj_vis}
We visualize the evolution of sampling trajectories in Figure \ref{fig:traj_evolution}. SDAC retains high curvature even at convergence ($200\mathrm{K}$), necessitating multi-step sampling. FLAME-R gradually straightens the flow, reducing truncation error for one-step inference. FLAME-M achieves straight trajectories almost immediately ($5\mathrm{K}$), demonstrating efficient training dynamics for one-step generation.
\begin{table}[]
    \centering
    \caption{\textbf{Efficiency Comparison.} We report the NFE, inference time (ms), frames per second (FPS), and training time (min). FLAME achieves comparable inference speed to SAC while maintaining competitive training times among generative policies.}
    \label{tab:inference_time}
    \vspace{2pt}
    \resizebox{\linewidth}{!}{
        \begin{tabular}{lccccc}
            \toprule
            \textbf{Algorithm} & \textbf{Type} & \textbf{NFE} & \textbf{Infer. (ms)} $\downarrow$ & \textbf{FPS} $\uparrow$ & \textbf{Train Time (min)} $\downarrow$ \\
            \midrule
            SAC   & Gaussian & 1 & $\mathbf{0.037}$ & 26869 &  26 \\
            \midrule
            SDAC  & Diffusion & 20 & $0.884$ & 1130 & 63 \\
            QVPO  & Diffusion & 20 & $0.874$ & 1144 & 88 \\
            DACER & Diffusion & 20 & $0.428$ & 2334 & 102 \\
            QSM   & Diffusion & 20 & $0.384$ & 2602 & 71 \\
            DIPO  & Diffusion & 20 & $0.245$ & 4083 & 121 \\
            \midrule
            \textbf{FLAME-R} & Flow & \textbf{1} & $0.134$ & 7475 & 72 \\
            \textbf{FLAME-M} & Flow & \textbf{1} & $0.085$ & 11682 & 107 \\
            \bottomrule
        \end{tabular}
    }
    \vspace{-20pt}
\end{table}
\subsection{Efficiency and Deployment}
\label{subsec:efficiency}
Table~\ref{tab:inference_time} compares computational efficiency using a single NVIDIA RTX 4090. In terms of inference, FLAME achieves an $8\times$ speedup over multi-step diffusion policies, matching the order-of-magnitude latency of SAC via one-step generation (NFE=1). Regarding training cost, FLAME-R remains competitive with standard diffusion baselines. Although FLAME-M incurs higher wall-clock training time due to the decoupled multi-step critic integration, this overhead is strictly confined to the learning phase, ensuring that deployment-time action generation remains unaffected and efficient.

FLAME introduces additional computation during policy improvement through Q-reweighting over $K$ reverse-sampled candidates, but the overhead remains practical in modern GPU settings. In our implementation, Q-reweighting is applied only during actor updates and is fully parallelized across candidates, and FLAME avoids backpropagation through multi-step denoising chains required by diffusion policies. As a result, the overall training cost is competitive with multi-step diffusion baselines while providing substantially lower inference latency at deployment.

FLAME targets settings that require expressive policies with low-latency deployment. FLAME-R is a robust default for stable training, while FLAME-M is preferred when deployment-time throughput is critical.

\section{Conclusion}
\label{sec:conclusion}

We presented FLAME, a MaxEnt RL framework for flow-based policies that achieves expressive behavior with one-step action generation. FLAME makes MaxEnt policy improvement tractable via a Q-reweighted flow-matching objective, and stabilizes entropy-regularized learning through two complementary likelihood estimators. Across MuJoCo and visual DeepMind Control Suite benchmarks, FLAME matches or improves upon strong multi-step diffusion baselines while reducing inference to a single function evaluation, enabling low-latency deployment. Future work includes extending FLAME to more complex real-world robotic domains.

\section*{Impact Statement}
This paper presents work whose goal is to advance the field of Machine
Learning. There are many potential societal consequences of our work, none of
which we feel must be specifically highlighted here.

\bibliography{example_paper}
\bibliographystyle{icml2026}

\newpage
\appendix
\onecolumn

\section{Related Work} \label{sec:related_work}

\paragraph{Diffusion-based Generative Policies.}
The integration of generative models into RL has shifted policy representations from unimodal Gaussians to expressive, multi-modal generators. Existing diffusion-based RL methods largely follow two optimization paradigms. Weighted-regression approaches formulate policy learning as weighted denoising regression on replay actions, typically avoiding backpropagation through the denoising chain; representative examples include QVPO~\cite{ding2024diffusion}, DIME~\cite{celik2025dimediffusionbasedmaximumentropyreinforcement}, and MaxEntDP~\cite{dong2025}, as well as reweighted score-matching objectives that yield efficient mirror-descent and max-entropy style updates such as DPMD and SDAC~\cite{ma2025effi}. Reparameterized policy-gradient approaches directly backpropagate return gradients through the reverse diffusion process, including DPPO~\cite{ren2024diffusionpolicypolicyoptimization} and diffusion actor-critic variants such as DACER~\cite{wang2024diffusion}. In addition, critic-guided diffusion optimization can be seen as a closely related instantiation of the weighted-regression route, where the denoiser is trained using critic-derived improvement signals, as in QSM~\cite{psenka2025learningdiffusionmodelpolicy} and DIPO~\cite{yang2023policyrepresentationdiffusionprobability}. Despite improved expressiveness, diffusion policies fundamentally require multi-step sampling at inference, making inference efficiency a persistent bottleneck and motivating one-step generative alternatives.

\paragraph{Flow-based Policies in Reinforcement Learning.}
Early approaches like FQL~\cite{park2025flowqlearning} prioritize stable policy extraction by distilling one-step policies from offline-trained flow models. This design effectively bypasses recursive backpropagation but may benefit from supplemental exploration strategies to further enhance performance in active online fine-tuning. Building upon this, methods such as ReinFlow~\cite{zhang2025reinflowfinetuningflowmatching} and Flow-GRPO~\cite{liu2025flowgrpotrainingflowmatching} have successfully extended flow matching to the online RL setting. ReinFlow introduces a stochastic discrete-time formulation that facilitates policy refinement through PPO-style surrogate objectives, while Flow-GRPO leverages a group-relative reward mechanism to optimize the flow vector field directly, bypassing the need for explicit log-probability computation during the training loop. In parallel, FPMD~\cite{chen2025onestepflowpolicymirror} reformulates MeanFlow to enable one-step generative policies trained via Policy Mirror Descent, allowing the agent to perform value-based policy updates directly in the action space. In a different vein, FlowRL~\cite{lv2025flowbasedpolicyonlinereinforcement} formulates a transport-regularized policy search objective , while ORW-CFM-W2~\cite{fan2025onlinerewardweightedfinetuningflow} utilizes Wasserstein-2 regularization within the conditional FM framework to ensure stable refinement. Closely related to the Maximum Entropy framework, MEOW~\cite{chao2024} adopts an explicit energy-based normalizing flow design, which jointly parameterizes the policy, soft Q-function, and soft value function to enable exact likelihood and value computation. While this approach provides a principled way to incorporate entropy, it necessitates explicit energy parameterization and normalization, which can impose architectural constraints and tightly couple policy learning to Bellman regression.

\paragraph{Summary and Positioning.}
Existing generative policies face a trade-off between the inference latency of diffusion models and the non-principled entropy handling of efficient flow-based surrogates. While explicit energy-based flows resolve this through joint modeling, they suffer from limited architectural flexibility. FLAME aims to combine the efficiency of one-step flow matching with the principled exploration of MaxEnt RL. By integrating a Q-reweighted objective for policy updates and a decoupled estimator for entropy calculation, we achieve accurate density estimation for stable training without sacrificing inference speed. This framework avoids the high latency of diffusion models and the complex constraints of energy-based architectures.

\section{Theoretical Derivations}
\label{app:derivations}

\subsection{Derivations of Proposition \ref{prop:fm_cfm_equivalence}}
\label{subsec:derivations_prop}

\textbf{Proposition \ref{prop:fm_cfm_equivalence}.}
\textit{Assuming $p_t(a \mid s) > 0$ for all $a \in \mathcal{A}$ and $t \in [0,1]$, the Conditional Flow Matching objective
$\mathcal{L}_{\mathrm{CFM}}$ and the marginal Flow Matching objective $\mathcal{L}_{\mathrm{FM}}$ satisfy
$\mathcal{L}_{\mathrm{CFM}}(\theta) = \mathcal{L}_{\mathrm{FM}}(\theta) + C_1$, where $C_1$ is independent of $\theta$.
Consequently, $\nabla_\theta \mathcal{L}_{\mathrm{FM}}(\theta) = \nabla_\theta \mathcal{L}_{\mathrm{CFM}}(\theta)$.}

\begin{proof}
Fix a state $s$ and time $t$. Let $u_\theta(a,t,s)$ be the learned velocity field, and $u_t(a \mid a_1,s)$ and $u_t(a \mid s)$ denote the conditional and marginal target vector fields, respectively.
Define the joint distribution
$p_t(a,a_1 \mid s) \triangleq p_t(a \mid a_1)\,\pi_{\text{new}}(a_1 \mid s)$
and the marginal
$p_t(a \mid s) = \int p_t(a \mid a_1)\pi_{\text{new}}(a_1 \mid s)\,\mathrm{d}a_1$.
Since $p_t(a\mid s)>0$, the conditional density
\[
p_t(a_1 \mid a,s) = \frac{p_t(a \mid a_1)\pi_{\text{new}}(a_1 \mid s)}{p_t(a \mid s)}
\]
is well-defined, and the marginal target field is the conditional expectation
\begin{equation}
\label{eq:ut_marginal_as_condexp_app}
u_t(a \mid s)
= \mathbb{E}_{a_1 \sim p_t(\cdot \mid a,s)}\big[\,u_t(a \mid a_1,s)\,\big]
= \int u_t(a \mid a_1,s)\,\frac{p_t(a \mid a_1)\pi_{\text{new}}(a_1 \mid s)}{p_t(a \mid s)}\,\mathrm{d}a_1.
\end{equation}

Recall the objectives (conditioning on $s$ is implicit):
\[
\mathcal{L}_{\mathrm{CFM}}(\theta)
= \mathbb{E}_{t}\,\mathbb{E}_{a_1 \sim \pi_{\text{new}}(\cdot \mid s)}\,\mathbb{E}_{a \sim p_t(\cdot \mid a_1)}
\left[\left\|u_\theta(a,t,s)-u_t(a \mid a_1,s)\right\|^2\right],
\]
\[
\mathcal{L}_{\mathrm{FM}}(\theta)
= \mathbb{E}_{t}\,\mathbb{E}_{a \sim p_t(\cdot \mid s)}
\left[\left\|u_\theta(a,t,s)-u_t(a \mid s)\right\|^2\right].
\]

Expanding the squared error in $\mathcal{L}_{\mathrm{CFM}}(\theta)$ and integrating over $(a,a_1)\sim p_t(a,a_1\mid s)$ yield
\begin{equation}
\begin{aligned}
\mathcal{L}_{\mathrm{CFM}}(\theta)
&= \mathbb{E}_{t}\iint
\Big(\|u_\theta\|^2 - 2\langle u_\theta, u_t(a\mid a_1,s)\rangle + \|u_t(a\mid a_1,s)\|^2\Big)\,
p_t(a \mid a_1)\pi_{\text{new}}(a_1 \mid s)\,\mathrm{d}a_1\,\mathrm{d}a \\
&= \mathbb{E}_{t}\int \|u_\theta\|^2\,p_t(a\mid s)\,\mathrm{d}a
- 2\,\mathbb{E}_{t}\int \left\langle u_\theta,
\int u_t(a\mid a_1,s)\,p_t(a_1\mid a,s)\,\mathrm{d}a_1
\right\rangle p_t(a\mid s)\,\mathrm{d}a \\
&\quad + \mathbb{E}_{t}\iint \|u_t(a\mid a_1,s)\|^2\,p_t(a \mid a_1)\pi_{\text{new}}(a_1 \mid s)\,\mathrm{d}a_1\,\mathrm{d}a.
\end{aligned}
\end{equation}
By \eqref{eq:ut_marginal_as_condexp_app}, the inner integral equals $u_t(a\mid s)$, so
\begin{equation}
\begin{aligned}
\mathcal{L}_{\mathrm{CFM}}(\theta)
&= \mathbb{E}_{t}\int \Big(\|u_\theta\|^2 - 2\langle u_\theta,u_t(a\mid s)\rangle\Big)p_t(a\mid s)\,\mathrm{d}a
+ \mathbb{E}_{t}\iint \|u_t(a\mid a_1,s)\|^2\,p_t(a \mid a_1)\pi_{\text{new}}(a_1 \mid s)\,\mathrm{d}a_1\,\mathrm{d}a .
\end{aligned}
\end{equation}

Similarly, expanding $\mathcal{L}_{\mathrm{FM}}(\theta)$ gives
\begin{equation}
\mathcal{L}_{\mathrm{FM}}(\theta)
= \mathbb{E}_{t}\int
\Big(\|u_\theta\|^2 - 2\langle u_\theta, u_t(a\mid s)\rangle + \|u_t(a\mid s)\|^2\Big)\,p_t(a\mid s)\,\mathrm{d}a.
\end{equation}

\begin{equation}
\label{eq:bridge_pointwise_fm_cfm}
\begin{aligned}
\left\|u_\theta(a,t,s)-u_t(a\mid s)\right\|^2
&=
\frac{1}{p_t(a\mid s)}
\int p_t(a\mid a_1)\,\pi_{\text{new}}(a_1\mid s)\,
\left\|u_\theta(a,t,s)-u_t(a\mid a_1,s)\right\|^2\,\mathrm{d}a_1 \\
&\quad -\Bigg(
\int p_t(a_1\mid a,s)\,\|u_t(a\mid a_1,s)\|^2\,\mathrm{d}a_1
-\|u_t(a\mid s)\|^2
\Bigg),
\end{aligned}
\end{equation}

\begin{equation}
\label{eq:bridge_LFM_LCFM}
\mathcal{L}_{\mathrm{FM}}(\theta)
=
\mathcal{L}_{\mathrm{CFM}}(\theta)
-
\mathbb{E}_{t}\!\int p_t(a\mid s)\,
\Bigg(
\int p_t(a_1\mid a,s)\,\|u_t(a\mid a_1,s)\|^2\,\mathrm{d}a_1
-\|u_t(a\mid s)\|^2
\Bigg)\mathrm{d}a.
\end{equation}

\begin{equation}
\label{eq:bridge_C_def}
C_1 \triangleq
\mathbb{E}_{t}\!\int p_t(a\mid s)\,
\Bigg(
\int p_t(a_1\mid a,s)\,\|u_t(a\mid a_1,s)\|^2\,\mathrm{d}a_1
-\|u_t(a\mid s)\|^2
\Bigg)\mathrm{d}a,
\qquad
\Rightarrow\quad
\mathcal{L}_{\mathrm{CFM}}(\theta)=\mathcal{L}_{\mathrm{FM}}(\theta)+C_1.
\end{equation}

Therefore,
\begin{equation}
\begin{aligned}
\mathcal{L}_{\mathrm{CFM}}(\theta) - \mathcal{L}_{\mathrm{FM}}(\theta)
&= \mathbb{E}_{t}\left[
\iint \|u_t(a\mid a_1,s)\|^2\,p_t(a \mid a_1)\pi_{\text{new}}(a_1 \mid s)\,\mathrm{d}a_1\,\mathrm{d}a
-\int \|u_t(a\mid s)\|^2\,p_t(a\mid s)\,\mathrm{d}a
\right],
\end{aligned}
\end{equation}
which is independent of $\theta$. Labeling the bracketed term as a constant $C_1$ completes the proof.
\end{proof}

\subsection{Proof of Proposition \ref{prop:reweighting_invariance}}
\label{app:proof_reweighting}

\begin{proof}
We prove the proposition by analyzing the objective function in the space of measurable vector fields.

Consider the $g$-weighted flow matching objective defined in formula~\eqref{eq:weighted_fm_prop}:
\begin{equation}
    \mathcal{L}_{\mathrm{FM}}^{g}(\theta) = \mathbb{E}_{t,s}\!\left[\int_{\mathcal{A}} g(a_t,s)\,\big\|u_\theta(a_t,t,s)-u_t(a_t\mid s)\big\|^2\,\mathrm{d}a_t\right].
\end{equation}
Since the weight function $g(a_t, s)$ is strictly positive (i.e., $g(a_t, s) > 0$ for all $a_t \in \mathcal{A}, s \in \mathcal{S}$) and the squared norm $\big\| \cdot \big\|^2$ is non-negative, the integrand is non-negative almost everywhere. Therefore, the lower bound of the objective is $\mathcal{L}_{\mathrm{FM}}^{g}(\theta) \geq 0$.

For a fixed time $t$ and state $s$, consider the functional $\mathcal{J}(u_\theta) = \int_{\mathcal{A}} g(a_t,s) \|u_\theta(a_t) - u_t(a_t)\|^2 \mathrm{d}a_t$. The global minimum $\mathcal{J}(u_\theta) = 0$ is achieved if and only if the integrand is zero almost everywhere (with respect to the Lebesgue measure on $\mathcal{A}$).
Specifically,
\begin{equation}
    g(a_t,s) \, \big\|u_\theta(a_t,t,s)-u_t(a_t\mid s)\big\|^2 = 0 \quad \text{a.e.}
\end{equation}
Because $g(a_t, s) > 0$, the above condition holds if and only if:
\begin{equation}
    \big\|u_\theta(a_t,t,s)-u_t(a_t\mid s)\big\|^2 = 0 \implies u_\theta(a_t,t,s) = u_t(a_t\mid s) \quad \text{a.e.}
\end{equation}
This is identical to the optimality condition for the unweighted objective (where $g(a_t, s) \equiv 1$).

Under the realizability assumption (i.e., assuming the parameter space $\Theta$ is expressive enough so that there exists $\theta^* \in \Theta$ satisfying $u_{\theta^*} = u_t$ almost everywhere), any parameter $\theta$ is a global minimizer of $\mathcal{L}_{\mathrm{FM}}^{g}(\theta)$ if and only if it satisfies $u_\theta = u_t$ almost everywhere. Thus, the set of global minimizers is invariant to the choice of the strictly positive weight function $g$.
\end{proof}

\subsection{Derivation of Q-Reweighted Flow Matching (FLAME-R)}
\label{subsec:derivations_flame_r}

We restate the main result. To handle the intractable partition function in the target distribution,
define a strictly positive reweighting function
\begin{equation}
g^{\text{MaxEnt}}(a_t, s) \triangleq h_t(a_t \mid s)\, Z(s)\, p_t(a_t \mid s),
\end{equation}
where $h_t(a_t \mid s)$ is a tractable proposal distribution with full support on $\mathcal{A}$.

Consider the (time-conditioned) weighted marginal FM loss
\begin{equation}
\label{eq:weighted_fm_app}
\mathcal{L}^{g}(\theta)
= \mathbb{E}_{t}\int g(a_t,s)\,\big\|u_\theta(a_t,t,s)-u_t(a_t\mid s)\big\|^2\,\mathrm{d}a_t.
\end{equation}
As in Proposition~\ref{prop:fm_cfm_equivalence}, using the law of total expectation and expanding (up to a
$\theta$-independent constant), one can rewrite \eqref{eq:weighted_fm_app} into the conditional form
\begin{equation}
\label{eq:weighted_fm_cond_app}
\mathcal{L}^{g}(\theta)
= \mathbb{E}_{t}\iint
\frac{g(a_t,s)}{p_t(a_t\mid s)}\,p_t(a_t\mid a_1)\,\pi_{\text{new}}(a_1\mid s)\,
\big\|u_\theta(a_t,t,s)-u_t(a_t\mid a_1)\big\|^2\,\mathrm{d}a_1\,\mathrm{d}a_t \;+\; C_2.
\end{equation}

Substituting
$g^{\text{MaxEnt}}(a_t,s)=h_t(a_t\mid s)Z(s)p_t(a_t\mid s)$ and
$\pi_{\text{new}}(a_1\mid s)=\exp(Q(s,a_1)/\alpha)/Z(s)$ into \eqref{eq:weighted_fm_cond_app} gives
\begin{equation}
\label{eq:after_cancel_app}
\mathcal{L}^{g_{\text{MaxEnt}}}(\theta)
= \mathbb{E}_{t}\iint
h_t(a_t\mid s)\,p_t(a_t\mid a_1)\,\exp\!\left(\frac{Q(s,a_1)}{\alpha}\right)\,
\big\|u_\theta(a_t,t,s)-u_t(a_t\mid a_1)\big\|^2\,\mathrm{d}a_1\,\mathrm{d}a_t \;+\; C_2.
\end{equation}

For the OT path $a_t=t a_1+(1-t)a_0$ with $a_0\sim\mathcal{N}(0,I)$, the forward kernel admits
\begin{equation}
p_t(a_t\mid a_1)=\mathcal{N}\!\left(a_t\mid t a_1,(1-t)^2I\right).
\end{equation}
Define the reverse conditional
\begin{equation}
\phi_{1|t}(a_1\mid a_t)
=\mathcal{N}\!\left(a_1\Bigm|\frac{a_t}{t},\frac{(1-t)^2}{t^2}I\right).
\end{equation}
A direct comparison of the Gaussian normalizing constants yields the \emph{exact} density relation
\begin{equation}
\label{eq:reverse_density_relation_app}
\phi_{1|t}(a_1\mid a_t)=t^{d}\,p_t(a_t\mid a_1)
\quad\Longleftrightarrow\quad
p_t(a_t\mid a_1)=t^{-d}\,\phi_{1|t}(a_1\mid a_t),
\end{equation}
where $d=\dim(\mathcal{A})$.

Combining \eqref{eq:reverse_density_relation_app} with \eqref{eq:after_cancel_app} yields
\begin{equation}
\begin{aligned}
\mathcal{L}^{g_{\text{MaxEnt}}}(\theta)
&= \mathbb{E}_{t}\iint
h_t(a_t\mid s)\,t^{-d}\phi_{1|t}(a_1\mid a_t)\,
\exp\!\left(\frac{Q(s,a_1)}{\alpha}\right)\,
\big\|u_\theta(a_t,t,s)-u_t(a_t\mid a_1)\big\|^2\,\mathrm{d}a_1\,\mathrm{d}a_t \;+\; C_2 \\
&= \mathbb{E}_{t}\Bigg[
t^{-d}\,\mathbb{E}_{\substack{a_t\sim h_t(\cdot\mid s)\\ a_1\sim \phi_{1|t}(\cdot\mid a_t)}}
\Big[
\exp\!\left(\frac{Q(s,a_1)}{\alpha}\right)\,
\big\|u_\theta(a_t,t,s)-u_t(a_t\mid a_1)\big\|^2
\Big]\Bigg] \;+\; C_2.
\end{aligned}
\end{equation}
For the OT path, $u_t(a_t\mid a_1)=a_1-a_0$ as well as $a_0=(a_t-t a_1)/(1-t)$ is deterministically recovered.

\paragraph{Dropping the $t^{-d}$ factor.}
The multiplicative factor $t^{-d}>0$ depends only on $t$ (and dimension $d$) and is independent of $\theta$.
In the realizable setting where the model can match the target vector field for each $t$, such a positive time-only reweighting
does not change the set of global minimizers (which achieve zero regression error pointwise).
In practice, $t^{-d}$ can heavily overweight small $t$ and cause numerical instability; hence we absorb it into the time weighting /
sampling design (together with $t\sim \mathcal{U}[\varepsilon,1]$) and omit it for clarity, yielding the proportional objective
reported in the main text. \qed

\subsection{Derivation of Q-Reweighted MeanFlow (QRMF)}
\label{subsec:derivations_flame_m}

We show the same Q-reweighting extends to the MeanFlow objective. Starting from the weighted MeanFlow loss
\begin{equation}
\mathcal{L}^g_{\mathrm{MF}}(\theta)
= \mathbb{E}_{t}\int g(a_t,s)\,
\left\| \overline{u}_{\theta}(a_t, \zeta, t, s) - \text{sg}(\overline{u}_{\text{tgt}}) \right\|^2 \mathrm{d}a_t.
\end{equation}
Introducing $\pi_{\text{new}}(a_1\mid s)$ and expanding the marginal (up to a $\theta$-independent constant) yield
\begin{equation}
\label{eq:mf_cond_app}
\mathcal{L}^g_{\mathrm{MF}}(\theta)
= \mathbb{E}_{t}\iint \frac{g(a_t, s)}{p_t(a_t \mid s)}\,
p_t(a_t \mid a_1)\,\pi_{\text{new}}(a_1 \mid s)\,
\left\| \overline{u}_{\theta} - \text{sg}(\overline{u}_{\text{tgt}}) \right\|^2 \mathrm{d}a_1 \mathrm{d}a_t \;+\; C_3.
\end{equation}

The substitution of $g^{\text{MaxEnt}}(a_t,s)=h_t(a_t\mid s)Z(s)p_t(a_t\mid s)$ and
$\pi_{\text{new}}(a_1\mid s)=\exp(Q(s,a_1)/\alpha)/Z(s)$ cancels $Z(s)$ and $p_t(a_t\mid s)$:
\begin{equation}
\label{eq:mf_after_cancel_app}
\mathcal{L}^{g_{\text{MaxEnt}}}_{\mathrm{MF}}(\theta)
= \mathbb{E}_{t}\iint h_t(a_t \mid s)\, p_t(a_t \mid a_1)\,
\exp\!\left(\frac{Q(s, a_1)}{\alpha}\right)
\left\| \overline{u}_{\theta} - \text{sg}(\overline{u}_{\text{tgt}}) \right\|^2 \mathrm{d}a_1 \mathrm{d}a_t \;+\; C_3.
\end{equation}

Using the same reverse density relation \eqref{eq:reverse_density_relation_app},
$p_t(a_t\mid a_1)=t^{-d}\phi_{1|t}(a_1\mid a_t)$, we obtain
\begin{equation}
\begin{aligned}
\mathcal{L}^{g_{\text{MaxEnt}}}_{\mathrm{MF}}(\theta)
&= \mathbb{E}_{t}\Bigg[
t^{-d}\,\mathbb{E}_{\substack{a_t \sim h_t(\cdot|s) \\ a_1 \sim \phi_{1|t}(\cdot|a_t)}}
\left[ \exp\!\left(\frac{Q(s, a_1)}{\alpha}\right)
\left\| \overline{u}_{\theta}(a_t, \zeta, t, s) - \text{sg}(\overline{u}_{\text{tgt}}) \right\|^2 \right]\Bigg] \;+\; C_3.
\end{aligned}
\end{equation}

As in the FLAME-R derivation, $t^{-d}$ depends only on $t$ and is independent of $\theta$.
We absorb it into the time weighting / sampling design and omit it in the final reported objective:
\begin{equation}
\mathcal{L}_{\text{QRMF}}(\theta) \propto
\mathbb{E}_{\substack{t \sim \mathcal{U}[\varepsilon,1] \\ a_t \sim h_t(\cdot|s) \\ a_1 \sim \phi_{1|t}(\cdot|a_t)}}
\left[ \exp\!\left(\frac{Q(s, a_1)}{\alpha}\right)
\left\| \overline{u}_{\theta}(a_t, \zeta, t, s) - \text{sg}(\overline{u}_{\text{tgt}}) \right\|^2 \right].
\end{equation}
Substituting the OT definitions into $\overline{u}_{\text{tgt}}$ yields the final algorithm. \qed

\section{Proofs for Density Estimation Bias}

\subsection{Proof of Proposition \ref{prop:single_step_bias} (Single-Step Entropy Bias)}
\label{proof:single_step_bias}

\begin{proof}
We analyze the one-step approximation that replaces the continuous-time change-of-variables integral by a first-order Jacobian term, which corresponds to a local (first-order) approximation of the log-determinant.
Consider the transformation from prior $a_0$ to action $a_1$ via a single step $a_1 = a_0 + \overline{u}_\theta(a_0, 0, 1, s)$ generated by the MeanFlow policy. The true change in log-density is determined by the Jacobian of the transformation $F(a) = a + \overline{u}_\theta(a, 0, 1, s)$. By the change of variables formula, we obtain:
\begin{equation}
    \Delta \log p_{\text{true}} = -\log |\det(\nabla_a F(a))| = -\log |\det(I + J_{\overline{u}})|,
\end{equation}
where we denote $J_{\overline{u}} = \frac{\partial \overline{u}_\theta(a, 0, 1, s)}{\partial a}$ as the Jacobian matrix of the global average velocity field. Using the identity $\log \det(A) = \text{Tr}(\log A)$ and the Taylor series expansion for the matrix logarithm $\log(I + X) = X - \frac{1}{2}X^2 + O(X^3)$, we expand the true log-density change:
\begin{equation}
\begin{aligned}
    \Delta \log p_{\text{true}} &= -\text{Tr}(\log(I + J_{\overline{u}})) \\
    &= -\text{Tr}\left( J_{\overline{u}} - \frac{1}{2}J_{\overline{u}}^2 + O(J_{\overline{u}}^3) \right) \\
    &= -\text{Tr}(J_{\overline{u}}) + \frac{1}{2}\text{Tr}(J_{\overline{u}}^2) - O(\|J_{\overline{u}}\|^3).
\end{aligned}
\end{equation}
The single-step augmented ODE estimator approximates this change using the instantaneous divergence integrated over the full interval:
\begin{equation}
    \Delta \log p_{\text{est}} = -\int_0^1 \text{Tr}(J_{\overline{u}})\, dt = -\text{Tr}(J_{\overline{u}}).
\end{equation}
The discretization error is the difference between the true value and this linear approximation:
\begin{equation}
    \mathcal{E}_{\text{single}} = \left| \Delta \log p_{\text{true}} - \Delta \log p_{\text{est}} \right| \approx \frac{1}{2}\text{Tr}(J_{\overline{u}}^2).
\end{equation}
Thus, the bias is dominated by the trace of the squared Jacobian, confirming the $O(\|J_{\overline{u}}\|^2)$ leading-order term.
\end{proof}

\subsection{Proof of Corollary \ref{cor:multi_step_error} (Multi-Step Error Suppression)}
\label{proof:multi_step_error}

\begin{proof}
To reduce the bias, the decoupled strategy in FLAME-M treats the mapping as a composition of $N_{\text{est}}$ discrete substeps. We partition the trajectory into $N_{\text{est}}$ intervals with step size $\Delta t = 1/N_{\text{est}}$. For each substep $k$, the integration moves from $t_k$ to $t_{k+1}$, where the update admits $a_{k+1} = a_k + \overline{u}_\theta(a_k, t_k, t_{k+1}, s) \Delta t$. 

The effective Jacobian for a single substep $k$ is $I + \Delta t J_{k}$, where $J_{k} = \frac{\partial \overline{u}_\theta(a_k, t_k, t_{k+1}, s)}{\partial a}$. Following the expansion in Proposition \ref{prop:single_step_bias}, the error $\mathcal{E}_k$ for this substep becomes:
\begin{equation}
    \mathcal{E}_k \approx \frac{1}{2}\text{Tr}((\Delta t J_k)^2) = \frac{1}{2 N_{\text{est}}^2} \text{Tr}(J_k^2).
\end{equation}
The total cumulative error $\mathcal{E}_{\text{multi}}$ over the full trajectory is the sum of errors from all $N_{\text{est}}$ steps. Assuming that the Jacobian magnitude is bounded such that $\text{Tr}(J_k^2) \approx \text{Tr}(J_{\overline{u}}^2)$:
\begin{equation}
    \mathcal{E}_{\text{multi}} = \sum_{k=1}^{N_{\text{est}}} \mathcal{E}_k \approx N_{\text{est}} \cdot \left( \frac{1}{2 N_{\text{est}}^2}\text{Tr}(J_{\overline{u}}^2) \right) = \frac{1}{2 N_{\text{est}}}\text{Tr}(J_{\overline{u}}^2).
\end{equation}
Consequently, by increasing $N_{\text{est}}$, the discretization error is suppressed linearly, allowing the critic to provide an accurate entropy signal for MaxEnt RL while interaction remains one-step.
\end{proof}

\section{Detailed Experimental Settings and Additional Results}
\label{app:settings}

\subsection{Implementation Details and Training Details}
\label{app:implementation_details}

\paragraph{Network Architecture.} 
We parameterize both the flow policy (actor) and the critic networks using Multi-Layer Perceptrons (MLPs) with Mish activations~\cite{misra2020mishselfregularizednonmonotonic}. To enable time-dependent flow generation, we encode the flow time $t$ using sinusoidal position embeddings~\cite{vaswani2023attentionneed} and concatenate the embedding vector with the state input before feeding it into the policy network. For visual tasks, we employ the same convolutional encoder as in ~\citet{kostrikov2021imageaugmentationneedregularizing, chen2025onestepflowpolicymirror} to extract features from image observations, which are then concatenated with the time embeddings and passed to the MLP policy network. The visual encoder is opdated using only the gradient from the critic loss Eq.~\eqref{eq:softbellman_errorvisual}.

\paragraph{Action Sampling.}
To enhance exploration and performance, the diffusion-based baselines in our experiments employ a batch action sampling strategy~\citep{ma2025effi,wang2024diffusion}. This involves sampling $N_{act}$ candidate actions $\{a_i\}_{i=1}^{N_{act}}$ from the current policy and selecting the one with the maximum Q-value: $a=\arg\max_i Q(s,a_i)$. In contrast, FLAME generates a single action directly from the flow model during both training and evaluation, without any candidate filtering or selection. Notably, during the evaluation phase, both FLAME-R and FLAME-M perform action generation in a single step (NFE=1) to demonstrate their inference efficiency.

\paragraph{Timestep Schedule.} 
The timestep schedule is critical for the learning dynamics of flow-based models. For FLAME-R, we sample the flow time $t$ from a standard uniform distribution $\mathcal{U}[\varepsilon, 1]$, $\varepsilon=10^{-3}$. For FLAME-M, we follow the sampling protocol in \citet{geng2025mean} but introduce a key modification to handle the non-stationary target distributions in online RL. Specifically, we enforce $\Pr(\zeta < t)=1$ by always sampling distinct time points, which forces the model to maintain integral consistency across the entire trajectory throughout the evolving policy iteration process.

\paragraph{Critic Training.} 
Following the approach in \citet{chen2025onestepflowpolicymirror}, we implement a dual-critic architecture to mitigate value overestimation during the policy evaluation step. For state-based tasks, the critics $Q_{\phi_{1,2}}$ are parameterized as standard MLPs; for visual-input tasks, we utilize a CNN encoder to process raw pixels. To incorporate entropy regularization, we estimate the log-likelihood $\log \pi_\theta(a|s)$ via the augmented ODE $N_{\text{est}}$ integration steps. The twin critics are updated by minimizing the soft Bellman error. For standard tasks, we use a single-step temporal difference target:
\begin{equation}
    L_{\phi_k}=\mathbb{E}_{s, a, s^{\prime}, a^{\prime}}\left[\left(Q_{\phi_k}(s, a)-\left(r+\gamma \min _{j=1,2} Q_{\bar{\phi}_j}\left(s^{\prime}, a^{\prime}\right) - \alpha \log \pi_\theta(a'|s') \right)\right)^2\right] \quad \forall k \in\{1,2\}.
\end{equation}
For visual tasks, we extend this to a multi-step $n$-step return:
\begin{equation}
    \label{eq:softbellman_errorvisual}
\begin{split}
    L_{\phi_k} = \mathbb{E}_{\{s_{t+i}, a_{t+i}\}_{i=0}^n \sim \mathcal{D}} \Big[ \Big( &Q_{\phi_k}(s_t, a_t) - \Big( \sum_{i=0}^{n-1} \gamma^i r_{t+i} \\
    & + \gamma^n (\min_{j=1,2} Q_{\bar{\phi}_j}(s_{t+n}, a_{t+n}) - \alpha \log \pi_\theta(a_{t+n}|s_{t+n})) \Big) \Big)^2 \Big] \quad \forall k \in\{1,2\},
\end{split}
\end{equation}

\paragraph{Training and Evaluation Protocol}
By default, we train for 1 million environment steps (corresponding to 200K gradient updates) on standard MuJoCo tasks. For the high-dimensional \textsc{Humanoid-v5} task, we extend the training budget to 2 million environment steps (400K gradient updates) to ensure convergence. Across all environments, we adopt an Update-To-Data (UTD) ratio of 0.2, performing one gradient update for every 5 environment steps collected. During evaluation, we report the average performance over 20 episodes.

\subsection{Hyperparameters}
\label{app:Hyperparameterssettings}

To ensure a fair comparison and reproducibility, we standardize the general reinforcement learning hyperparameters across all algorithms. We separate the configuration into three tables: Table \ref{tab:common_hyperparams} and Table~\ref{tab:visual_hyperparams} list the common settings used for all methods on MuJoCo benchmarks and high-dimensional visual control tasks, respectively. Table \ref{tab:flame_specific_hyperparams} details the specific parameters for our FLAME algorithm.

\begin{table}[p]
    \small
\centering
\caption{General Hyperparameters for state-based MuJoCo environments}
\label{tab:common_hyperparams}
\begin{tabular}{ll}
\toprule
\textbf{Hyperparameter} & \textbf{Value} \\
\midrule
Critic learning rate & $3 \times 10^{-4}$ \\
Policy learning rate & $3 \times 10^{-4}$, linear annealing to $3 \times 10^{-5}$ \\
Value network hidden layers & 3 \\
Value network hidden neurons & 256 \\
Value network activation & Mish \\
Policy network hidden layers & 3 \\
Policy network hidden neurons & 256 \\
Policy network activation & Mish \\
Batch size & 256 \\
Replay buffer size & $1 \times 10^6$ \\
Action repeat & 1 \\
Frame stack & 1 \\
n-step returns & 1 \\
Discount factor & 0.99 \\
\bottomrule
\end{tabular}
\end{table}

\begin{table}[p]
    \small
\centering
\caption{General Hyperparameters for visual observation DMControl environments}
\label{tab:visual_hyperparams}
\begin{tabular}{ll}
\toprule
\textbf{Hyperparameter} & \textbf{Value} \\
\midrule
Critic learning rate & $3 \times 10^{-4}$ \\
Policy learning rate & $3 \times 10^{-4}$, linear annealing to $3 \times 10^{-5}$ \\
Value network hidden layers & 3 \\
Value network hidden neurons & 256 \\
Value network activation & Mish \\
Policy network hidden layers & 3 \\
Policy network hidden neurons & 256 \\
Policy network activation & Mish \\
Batch size & 256 \\
Replay buffer size & $1 \times 10^6$ \\
Encoder network convolutional layers & 4 \\
Encoder network kernel size & $3 \times 3$ \\
Encoder network activation & ReLU \\
Replay buffer size & $1 \times 10^6$ \\
Action repeat & 2 \\
Frame stack & 3 \\
n-step returns & 3 \\
\bottomrule
\end{tabular}
\end{table}

\begin{table}[p]
    \small
\centering
\caption{FLAME-Specific Hyperparameters}
\label{tab:flame_specific_hyperparams}
\begin{tabular}{ll}
\toprule
\textbf{Parameter} & \textbf{Value} \\
\midrule
Sampling stepsize (Training, FLAME-R) & $dt=0.05$ ($N_{\text{gen}}=20$) \\
Sampling stepsize (Evaluation) & $dt=1.0$ ($N_{\text{gen}}=1$) \\
Importance Samples ($K$) & 300 \\
Entropy Estimation Steps ($N_{\text{est}}$) & 5 \\
Coupling Strategy & Optimal Transport \\
Training Horizon & $t \sim \mathcal{U}[\varepsilon, 1]$, $\varepsilon=10^{-3}$ \\
ODE Solver & Euler \\
\bottomrule
\end{tabular}
\end{table}

\subsection{Baseline Implementation Details}
\label{app:baseline_details}

To ensure a rigorous evaluation, we compare FLAME against a diverse set of baselines across three distinct categories. 

\textbf{Classic Model-Free RL.} We include three representative algorithms: SAC~\citep{haarnoja2019softactorcriticalgorithmsapplications}, PPO~\citep{schulman2017proximalpolicyoptimizationalgorithms}, and TD3~\citep{pmlr-v80-fujimoto18a}. These methods utilize Gaussian policy parameterizations and operate with single-step inference (NFE=1). We utilize standard PyTorch implementations with default hyperparameters aligned with the original papers. For PPO, we maintain a replay buffer size of 4096 and perform 10 gradient updates per collected batch.

\textbf{Diffusion-Based RL.} This category represents high-expressivity policies that typically require iterative sampling. We compare against state-of-the-art methods including SDAC~\citep{ma2025effi}, DPMD~\citep{ma2025effi}, DIPO~\citep{yang2023policyrepresentationdiffusionprobability}, DACER~\citep{wang2024diffusion}, QSM~\citep{psenka2025learningdiffusionmodelpolicy}, and QVPO~\citep{ding2024diffusion}. To guarantee high-fidelity sample quality, the diffusion backbone is configured with $T=20$ denoising steps during inference (NFE=20), consistent with the standard settings in their respective literature.

\textbf{Flow-Based RL.} We evaluate recent baselines targeting efficient inference, including FPMD~\citep{chen2025onestepflowpolicymirror}, FlowRL~\citep{lv2025flowbasedpolicyonlinereinforcement}, and MEOW~\citep{chao2024}. For FPMD, we utilize the open-source JAX implementation to evaluate its Mirror Descent-based generative variants, FPMD-R and FPMD-M. For FlowRL and MEOW, we employ their respective PyTorch implementations, which leverage transport-regularized policy search and explicit energy-based normalizing flows. All methods adopt the optimal hyperparameter settings reported in their original papers to ensure a fair comparison.
\begin{figure}[]
    \centering
    \setlength{\tabcolsep}{1pt} 
    \renewcommand{\arraystretch}{0.5}

    \includegraphics[width=0.19\textwidth]{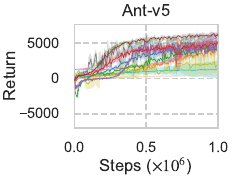} \hfill
    \includegraphics[width=0.19\textwidth]{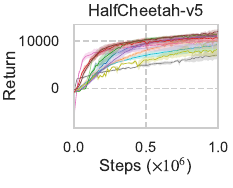} \hfill
    \includegraphics[width=0.19\textwidth]{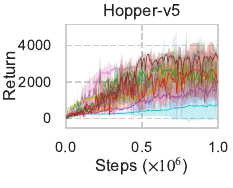} \hfill
    \includegraphics[width=0.19\textwidth]{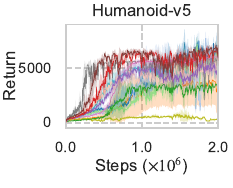} \hfill
    \includegraphics[width=0.19\textwidth]{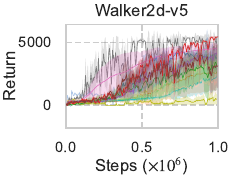}
    
    \vspace{0.15cm}

    \includegraphics[width=0.19\textwidth]{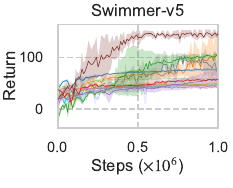} \hfill
    \includegraphics[width=0.19\textwidth]{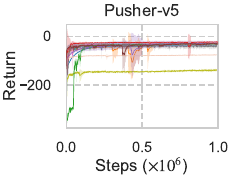} \hfill
    \includegraphics[width=0.19\textwidth]{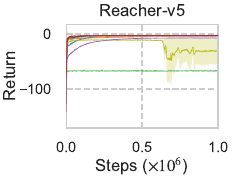} \hfill
    \includegraphics[width=0.19\textwidth]{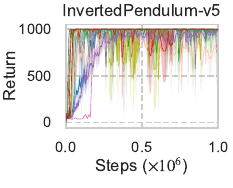} \hfill
    \includegraphics[width=0.19\textwidth]{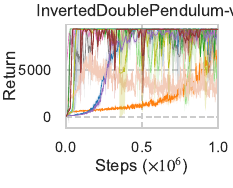}
    \vspace{0.2cm}

    \includegraphics[width=1\textwidth]{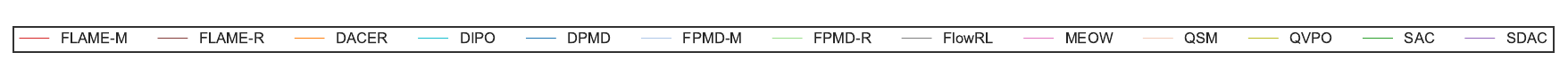}

    \caption{\textbf{Training curves on 10 MuJoCo continuous control benchmarks.} 
    The x-axis represents environment steps ($\times 10^6$) and the y-axis shows the average episode return. 
    Our methods, FLAME-R and FLAME-M, consistently match or outperform both Gaussian baselines (SAC) and multi-step diffusion policies (SDAC) while requiring only one-step inference. Shaded regions denote the standard deviation across 5 seeds.}
    \label{fig:training_curves_appendix}
\end{figure}

\subsection{Full Training Curves on MuJoCo Benchmarks}
\label{app:training_curves}

We present the complete training curves for all 10 MuJoCo environments evaluated in the main paper. As shown in Figure \ref{fig:training_curves_appendix}, both FLAME-R and FLAME-M demonstrate stable convergence and consistently high asymptotic performance compared to baselines.

\subsection{Evaluations on Visual RL Tasks}
\label{app:visual_rl}

We further evaluate FLAME on four visual-input continuous control tasks from the DeepMind Control Suite (DMC)~\citep{tassa2018deepmindcontrolsuite} to verify its scalability to high-dimensional observation spaces. We compare FLAME against three representative baselines: SAC~\citep{haarnoja2017reinforcement} (Gaussian), DPMD~\citep{ma2025effi} (Diffusion), and FPMD~\citep{chen2025onestepflowpolicymirror} (Flow). To handle pixel inputs, the SAC implementation follows DrQ~\citep{kostrikov2021imageaugmentationneedregularizing}, while DPMD and FPMD are implemented following the approach in \citet{chen2025onestepflowpolicymirror}.

As summarized in Table~\ref{tab:visual_results} and Figure~\ref{fig:visual_curves}, FLAME exhibits strong performance across the evaluated benchmarks. In the challenging Dog domain, FLAME-R achieves higher average returns than Gaussian policies and consistently exceeds the performance of the multi-step diffusion policy DPMD as well as the flow-based FPMD. These results suggest that our Q-reweighted objective and decoupled entropy estimation scale effectively to latent representations, enabling reliable one-step control directly from pixels.

\begin{figure}[H]
    \centering
    \begin{minipage}{0.24\textwidth}
        \centering
        \includegraphics[width=\linewidth]{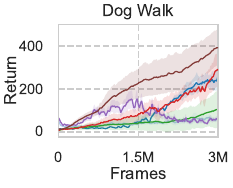}
    \end{minipage} \hfill
    \begin{minipage}{0.24\textwidth}
        \centering
        \includegraphics[width=\linewidth]{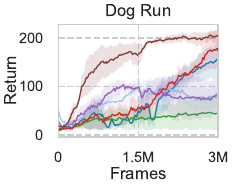}
    \end{minipage} \hfill
    \begin{minipage}{0.24\textwidth}
        \centering
        \includegraphics[width=\linewidth]{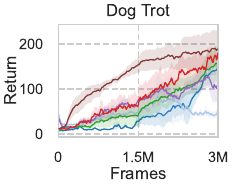}
    \end{minipage} \hfill
    \begin{minipage}{0.24\textwidth}
        \centering
        \includegraphics[width=\linewidth]{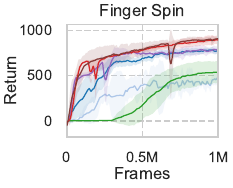}
    \end{minipage}

    \vspace{0.2cm} 

    \includegraphics[width=0.8\linewidth]{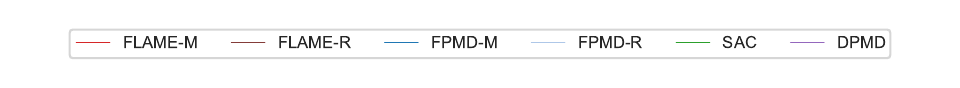} 

    \caption{\textbf{Training curves on visual DMC tasks.} FLAME-R (solid red) consistently matches or exceeds the performance of DPMD and FPMD baselines across all pixel-based environments.}
    \label{fig:visual_curves}
\end{figure}

\begin{table}[H]
\centering
\caption{Performance on visual-input tasks from the DeepMind Control Suite. Results are reported as mean $\pm$ std across 5 seeds. \colorbox{lightblue}{Blue} highlights the best overall performance.}
\label{tab:visual_results}
\small
\setlength{\tabcolsep}{5pt}
\begin{tabular}{lcccccc}
\toprule
\textbf{Task} & \textbf{SAC} & \textbf{DPMD} & \textbf{FPMD-R} & \textbf{FPMD-M} & \textbf{FLAME-M} & \textbf{FLAME-R} \\
\midrule
\textsc{Dog-walk}    & $104.0 \pm 85.7$ & $58.7 \pm 9.7$  & $58.3 \pm 3.2$  & $253.1 \pm 49.7$ & $288.1 \pm 77.7$ & \cellcolor{lightblue}$391.1 \pm 119.3$ \\
\textsc{Dog-run}     & $44.7 \pm 35.2$  & $83.6 \pm 18.3$ & $81.1 \pm 11.2$ & $156.1 \pm 19.0$ & $176.5 \pm 23.9$ & \cellcolor{lightblue}$204.6 \pm 4.9$   \\
\textsc{Dog-trot}    & $156.7 \pm 38.2$ & $107.5 \pm 30.6$ & $48.9 \pm 4.4$  & $144.4 \pm 8.9$  & $174.0 \pm 12.8$ & \cellcolor{lightblue}$187.8 \pm 53.6$ \\
\textsc{Finger-spin} & $538.4 \pm 121.1$& $783.6 \pm 43.5$ & $474.2 \pm 39.3$ & $775.5 \pm 29.7$ & $897.5 \pm 36.0$ & \cellcolor{lightblue}$908.7 \pm 14.6$  \\
\bottomrule
\end{tabular}
\end{table}

\subsection{Toy Experiment Details and Additional Visualizations}
\label{app:toy_details}

\textbf{2D GMM Environment.} 
This setup isolates density-estimation bias by providing an analytical ground-truth $\log p$. We benchmark the log-likelihood estimator using a 2D GMM comprising four symmetric modes centered at $\mu \in \{(\pm 1.0, \pm 1.0)\}$ with $\sigma=0.5$ and uniform weights. The policy network (3-layer MLP, 256 units per layer, Mish activation) is trained via OT-CFM for 10k iterations. We evaluate density estimation fidelity by comparing the trace-based log-likelihood (Eq.~\ref{eq:aug_ode}) against the analytical density across varying integration steps $N_{\text{est}}$.

Figure~\ref{fig:toy_density_vis} illustrates the impact of the integration schedule. At $N_{\text{est}}=1$, both FLAME-R and FLAME-M exhibit severe density blurring, failing to resolve the distinct multi-modal structure. This observation corroborates the $O(\|J\|^2)$ leading-order error term discussed in Proposition~\ref{prop:single_step_bias}. Increasing the estimator to $N_{\text{est}}=5$ effectively eliminates this bias, accurately recovering the Ground Truth modes. These results validate our decoupled integration strategy: employing $N_{\text{est}}=1$ for low-latency inference while maintaining $N_{\text{est}} \ge 5$ to provide stable, accurate signals for critic updates.

\textbf{MultiGoal Environment.} 
The state space is 2D continuous coordinates $s \in \mathbb{R}^2$. The action space controls velocity $a \in [-1, 1]^2$. The reward function is a mixture of unnormalized Gaussians: $r(s, a) = \sum_{i=1}^4 \exp(-\|s - g_i\|^2 / \sigma^2)$, where goals $g_i \in \{(\pm 5, \pm 5)\}$. This task requires the agent to discover all modes from scratch without warm-up demonstrations.

\begin{figure*}[t]
    \centering
    \def\colwidth{0.15\textwidth} 
    
    \begin{minipage}[c]{0.19\textwidth}
        \centering
        \includegraphics[width=\linewidth]{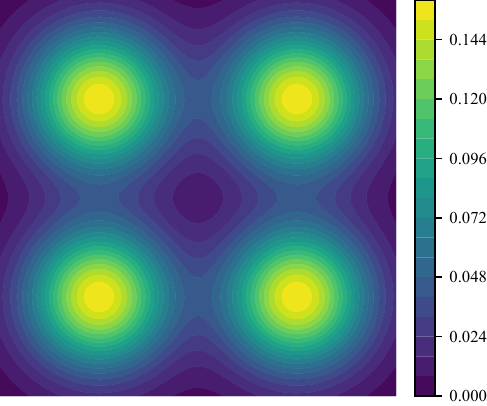} \\ 
        \vspace{2pt}
        \scriptsize \textbf{Ground Truth}
    \end{minipage}
    \hfill 
    \begin{minipage}[c]{0.80\textwidth}
        \centering
        \begin{tabular}{ccccc}
             \includegraphics[width=\colwidth]{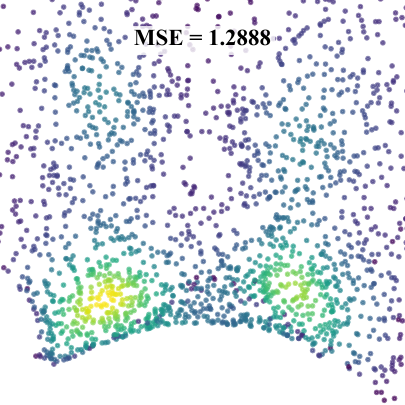} & 
             \includegraphics[width=\colwidth]{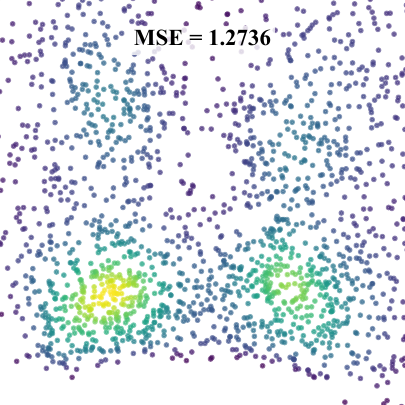} &
             \includegraphics[width=\colwidth]{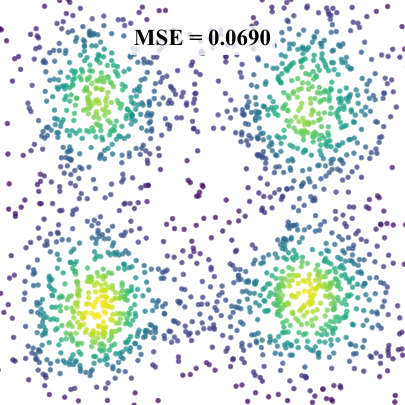} &
             \includegraphics[width=\colwidth]{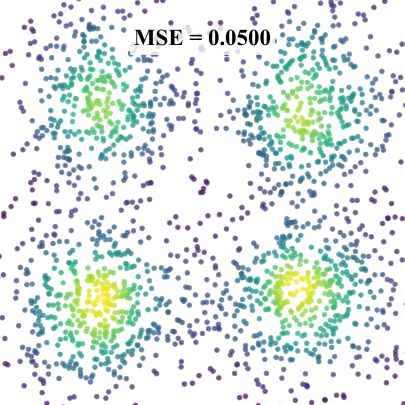} &
             \includegraphics[width=\colwidth]{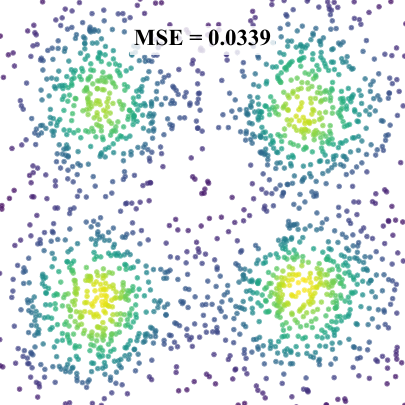} \\
             \tiny FLAME-R ($N_{\text{est}}=1$) & \tiny FLAME-R ($N_{\text{est}}=2$) & \tiny FLAME-R ($N_{\text{est}}=5$) & \tiny FLAME-R ($N_{\text{est}}=10$) & \tiny FLAME-R ($N_{\text{est}}=20$)
        \end{tabular}
        
        \vspace{0.3em}
        
        \begin{tabular}{ccccc}
             \includegraphics[width=\colwidth]{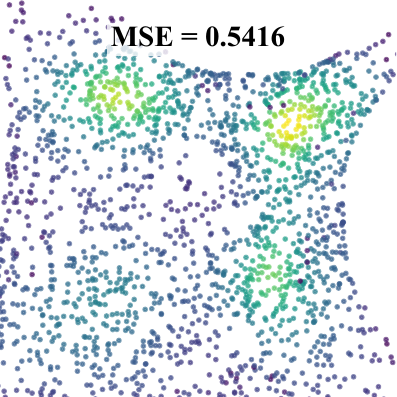} & 
             \includegraphics[width=\colwidth]{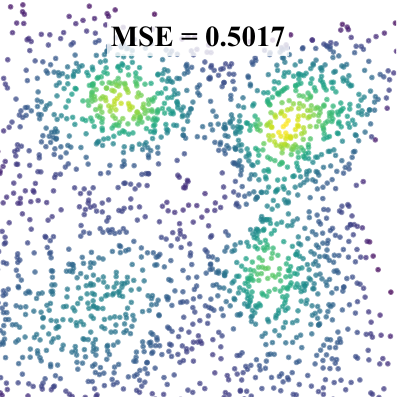} &
             \includegraphics[width=\colwidth]{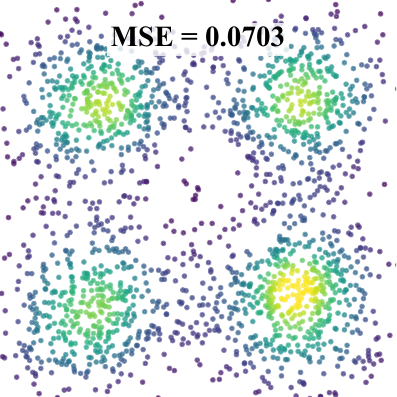} &
             \includegraphics[width=\colwidth]{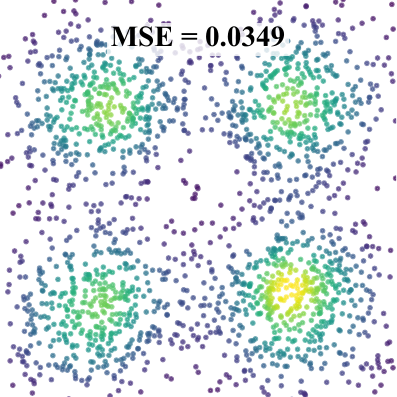} &
             \includegraphics[width=\colwidth]{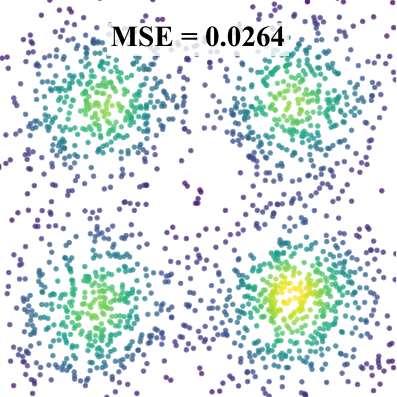} \\
             \tiny FLAME-M ($N_{\text{est}}=1$) & \tiny FLAME-M ($N_{\text{est}}=2$) & \tiny FLAME-M ($N_{\text{est}}=5$) & \tiny FLAME-M ($N_{\text{est}}=10$) & \tiny FLAME-M ($N_{\text{est}}=20$)
        \end{tabular}
    \end{minipage}
    
    \caption{\textbf{Visualizing the Bias-Variance Trade-off in Entropy Estimation.} 
    \textbf{Left:} The Ground Truth density of the 4-mode GMM. 
    \textbf{Right:} Reconstructed densities using FLAME-R (Top) and FLAME-M (Bottom) with varying ODE steps $N_{\text{est}} \in \{1, 2, 5, 10, 20\}$. 
    At $N_{\text{est}}=1$, both methods blur the modes significantly. At $N_{\text{est}}=5$, the estimation becomes sharp and accurate, matching the Ground Truth. Further increasing to $N_{\text{est}}=20$ yields minimal visual improvement, confirming that $N_{\text{est}}=5 \sim 10$ is sufficient for the critic.}
    \label{fig:toy_density_vis}
\end{figure*}
\begin{figure*}[t]
    \centering
    \setlength{\tabcolsep}{1pt}
    
    \begin{minipage}{0.24\textwidth}
        \centering
        \includegraphics[width=\linewidth]{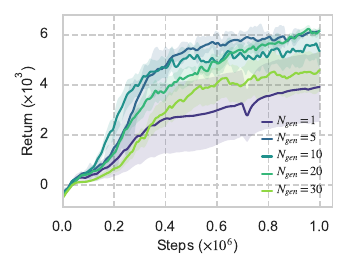}
        \caption*{\scriptsize (a) Action Steps ($N_{\text{gen}}$)}
    \end{minipage}
    \hfill
    \begin{minipage}{0.24\textwidth}
        \centering
        \includegraphics[width=\linewidth]{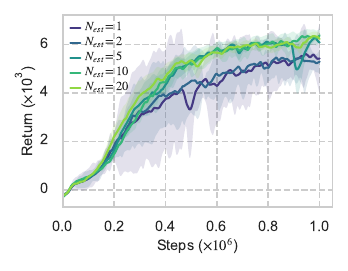}
        \caption*{\scriptsize (b) Entropy Steps ($N_{\text{est}}$)}
    \end{minipage}
    \hfill
    \begin{minipage}{0.24\textwidth}
        \centering
        \includegraphics[width=\linewidth]{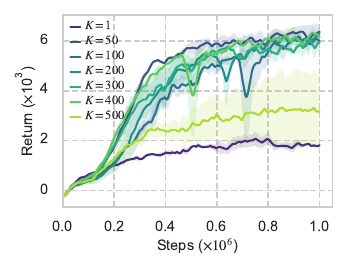}
        \caption*{\scriptsize (c) Sample Count ($K$)}
    \end{minipage}
    \hfill
    \begin{minipage}{0.24\textwidth}
        \centering
        \includegraphics[width=\linewidth]{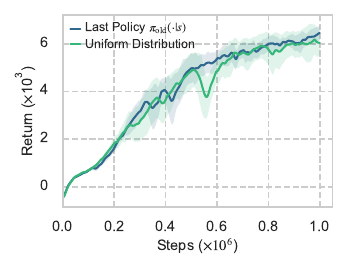}
        \caption*{\scriptsize (d) Proposal $h_t$}
    \end{minipage}
    
    \caption{\textbf{Sensitivity Analysis on Ant-v5.} 
    \textbf{(a)} Impact of flow integration steps for action generation.
    \textbf{(b)} Impact of integration steps for entropy estimation ($N_{\text{est}}$).
    \textbf{(c)} Impact of importance sample count $K$.
    \textbf{(d)} Impact of different proposal distributions $h_t$.}
    \label{fig:sensitivity}
\end{figure*}

\subsection{Sensitivity Analysis}
\label{app:ablation}
To evaluate the robustness of FLAME, we conduct ablation studies on the \textsc{Ant} benchmark, focusing on four critical hyperparameters: the number of flow steps ($N_{\text{gen}}$), entropy estimation steps ($N_{\text{est}}$), importance sample count ($K$), and the choice of proposal distribution ($h_t$).

\textbf{Effect of Number of Flow Steps ($N_{\text{gen}}$).} 
We investigate how the integration fidelity of the probability flow affects control performance. As illustrated in Figure \ref{fig:sensitivity}(a), although using fewer flow steps (e.g.,$N_{\text{gen}}=1$) exhibits higher variance during the early stages of training, FLAME-R effectively straightens the probability flow over time. This enables the policy to maintain high-fidelity control even with a single flow step at deployment. We utilize $N_{\text{gen}}=20$ during training.

\textbf{Effect of Entropy Estimation Steps ($N_{\text{est}}$).} 
We analyze the impact of the integration steps $N_{\text{est}}$ used in the augmented ODE for log-likelihood computation. As illustrated in Figure \ref{fig:sensitivity}(b), a single-step estimation ($N_{\text{est}}=1$) brings significant bias, leading to suboptimal asymptotic returns. Increasing $N_{\text{est}}$ effectively reduces this bias, with performance improvement significantly up to $N_{\text{est}}=10$. While $N_{\text{est}}=20$ yields similar results, it incurs higher computational cost; thus, we find $N_{\text{est}}=5 \sim 10$ to be the optimal trade-off for the critic's update.

\textbf{Effect of Importance Sample Count ($K$).} 
The Q-Weighted objective relies on a batch of $K$ candidate samples to approximate the optimal policy via importance sampling. Figure \ref{fig:sensitivity}(c) shows that performance is relatively robust for $K \in [50, 400]$. Extremely small values (e.g., $K=1$) fail to capture the target distribution, while excessively large values (e.g., $K=500$) do not yield further gains. We select $K=300$ as the default.

\textbf{Effect of Proposal Distribution ($h_t$).} 
We examine the impact of the proposal distribution $h_t(a_t|s)$ used in the Q-reweighted flow matching objective. According to Proposition \ref{prop:reweighting_invariance}, the set of global minimizers is invariant to the choice of any strictly positive weighting function. To verify this, we compared two variants: a \textit{Uniform Distribution} over the action space and the \textit{Last Policy} $\pi_{\text{old}}(\cdot|s)$. As shown in Figure \ref{fig:sensitivity}(d), both distributions show nearly identical performance. This empirical result confirms the theoretical robustness of our reweighting scheme, allowing for flexible choices of $h_t$ in practical implementations.
\end{document}